%% file: ShapBPT_Arxiv.tex
\documentclass[twoside]{article}

\usepackage{PRIMEarxiv}

\usepackage[utf8]{inputenc} 
\usepackage[T1]{fontenc}    
\usepackage{hyperref}       
\usepackage{url}            
\usepackage{booktabs}       
\usepackage{amsfonts}       
\usepackage{nicefrac}       
\usepackage{microtype}      
\usepackage{lipsum}
\usepackage{fancyhdr}       
\usepackage{graphicx}       
\graphicspath{{media/}}     

\usepackage{amsmath}
\usepackage{amssymb}
\newcommand{\tuple}[1]{{\langle{#1}\rangle}}

\usepackage{amsthm}

\newcommand{\setN}{\mathcal{N}}

\newcommand{\TT}{\mathcal{T}}

\mathchardef\mhyphen="2D
\newcommand{\AUIoU}{\ensuremath{\mathit{AU{\mhyphen}IoU}}}
\newcommand{\maxIoU}{\ensuremath{\mathit{max{\mhyphen}IoU}}}

\newtheorem{theorem}{Theorem}
\theoremstyle{plain}
\newtheorem{example}{\normalfont\textit{Example}}

\usepackage[linesnumbered,ruled,vlined]{algorithm2e}
\DontPrintSemicolon
\SetInd{0.3em}{0.7em}
\renewcommand{\KwSty}[1]{\textnormal{\textcolor{blue!90!black}{\bfseries #1}}\unskip}

\SetKwComment{Comment}{\color{green!50!black}/\!/\,}{}
\SetKwComment{tcp}{/\!/\,}{}

\newcommand{\algrule}[1][.5pt]{\par\vskip.1\baselineskip\hrule height #1\par\vskip.3\baselineskip} 

\renewcommand{\FuncSty}[1]{\textnormal{\textcolor[rgb]{0.3,0.1,0.4}{\bfseries #1}}}
\newcommand{\FuncCall}[2]{\FuncSty{#1}(#2)}
\SetKwProg{Function}{function}{}{}

\let\oldnl\nl
\newcommand{\nonl}{\renewcommand{\nl}{\let\nl\oldnl}}

\SetKw{This}{this}
\SetKw{None}{none}
\SetKw{True}{T}
\SetKw{False}{F}

\usepackage{xcolor}

\newenvironment{links}{%
  \begin{itemize}
  \setlength{\itemsep}{0pt}
}{%
  \end{itemize}
}

\newcommand{\link}[2]{%
  \item \textbf{#1:} \href{#2}{\texttt{#2}}
}

\usepackage{float}

\usepackage{etoc}

\title{ShapBPT: Image Feature Attributions Using Data-Aware Binary Partition Trees
\thanks{Published in AAAI 2026. Please cite the published version.}

\author{
Muhammad Rashid$^{1}$, Elvio G.\ Amparore$^{1}$, Enrico Ferrari$^{2}$, Damiano Verda$^{2}$\\[3pt]
$^{1}$University of Torino, Computer Science Department, Torino, Italy\\
$^{2}$Rulex Innovation Labs, Genova, Italy\\[2pt]
\texttt{\{muhammad.rashid, elviogilberto.amparore\}@unito.it}\\
\texttt{\{enrico.ferrari, damiano.verda\}@rulex.ai}
}

}

\date{March 2026 \\ ArXiv preprint}

\begin{document}
\maketitle

\vspace{-10pt}

\pagestyle{fancy}
\fancyhf{} 

\fancyhead[RO]{Published at the Proceedings of The 40th Annual AAAI Conference on Artificial Intelligence}

\fancyhead[LE]{\textit{ShapBPT: Image Feature Attributions Using Data-Aware Binary Partition Trees}}

\fancyfoot[C]{\thepage}

\renewcommand{\headrulewidth}{0.4pt}
\renewcommand{\footrulewidth}{0pt}

\thispagestyle{fancy}

\vspace{-20pt}

\begin{abstract}
Pixel-level feature attributions are an important tool in eXplainable AI for Computer Vision (XCV), providing visual insights into how image features influence model predictions. The Owen formula for hierarchical Shapley values has been widely used to interpret machine learning (ML) models and their learned representations. However, existing hierarchical Shapley approaches do not exploit the multiscale structure of image data, leading to slow convergence and weak alignment with the actual morphological features. Moreover, no prior Shapley method has leveraged data-aware hierarchies for Computer Vision tasks, leaving a gap in model interpretability of structured visual data.

To address this, this paper introduces ShapBPT, a novel data-aware XCV method based on the hierarchical Shapley formula. 
ShapBPT assigns Shapley coefficients to a multiscale hierarchical structure tailored for images, the Binary Partition Tree (BPT). 
By using this data-aware hierarchical partitioning, ShapBPT ensures that feature attributions align with intrinsic image morphology, effectively prioritizing relevant regions while reducing computational overhead.
This advancement connects hierarchical Shapley methods with image data, providing a more efficient and semantically meaningful approach to visual interpretability. 
Experimental results confirm ShapBPT’s effectiveness, demonstrating superior alignment with image structures and improved efficiency over existing XCV methods, and a 20-subject user study confirming that ShapBPT explanations are preferred by humans.
\end{abstract}


\begin{links}
  \link{Code}{https://github.com/amparore/shap\_bpt}
  \link{Tests}{https://github.com/rashidrao-pk/shap\_bpt\_tests}
  \link{Python Package}{https://pypi.org/project/shap-bpt/}\footnote{pip install shap-bpt}
  
\end{links}

\small
\textbf{Preprint version.} This paper was published in the Proceedings of the AAAI Conference on Artificial Intelligence (AAAI 2026).

\vspace{2pt}
\textit{Please cite the published version:}\\
Rashid, M., Amparore, E. G., Ferrari, E., \& Verda, D. (2026). 
\emph{ShapBPT: Image Feature Attributions Using Data-Aware Binary Partition Trees}. 
Proceedings of the AAAI Conference on Artificial Intelligence, 40(30), 25099--25107. \\
\url{https://doi.org/10.1609/aaai.v40i30.39699}

\vspace{4pt}
\textbf{ArXiv version note.} Compared with the AAAI proceedings version, this manuscript includes additional implementation details, extended experimental discussion, and the full technical appendix.

\section{Introduction}
\label{sec:intro}


A fundamental challenge in Machine Learning (ML) for Computer Vision is explaining how a black-box model classifies images, providing insights into the representations the model has learned from data. A key approach to this problem involves attributing importance scores to individual pixels, identifying their contribution to the model’s decision-making process. This task, commonly referred to as \emph{explaining model predictions}, plays a crucial role in enhancing interpretability and trust in AI-driven image classification.
One of the most widely used methods for this purpose is SHAP (SHapley Additive exPlanations), which applies game-theoretic principles to ML explainability. SHAP combines feature removal (masking)\,\cite{lundberg2017unified} with hierarchical image partitioning\,\cite{shapPartitionExplainer}, computing feature attributions over a refinable axis-aligned (AA) grid of pixels to approximate the regions most relevant to an image classifier.
Another influential method is LIME (Local Interpretable Model-agnostic Explanations)\,\cite{ribeiro2016lime}, which, despite lacking theoretical guarantees, remains popular for its ability to pre-identify relevant image regions through segmentation. 
However, LIME and similar approaches rely on predefined segmentation matching the relevant image regions, and they cannot adaptively refine these regions if the initial segmentation is inadequate, limiting their effectiveness for complex image data.

Since models learn to recognize structured patterns from image data, an image classifier is expected to base its decisions on a hierarchical representation that captures distinct morphological characteristics—such as shape, texture, and color continuity—of the classified objects.  
A key challenge lies therefore in integrating theoretically sound attribution methods, such as Shapley coefficients, with data-aware image hierarchies. 
Computing Shapley coefficients over adaptive, data-driven hierarchical partitions can enhance interpretability by aligning attributions more closely with the model’s learned representations. 
However, for this approach to be effective, the partitions must remain flexible and refinable, rather than being imposed a priori (as done by LIME or similar approaches).

\vspace{4pt} 
\noindent 
This paper provides the following contributions:
\begin{enumerate}
    \item A novel hierarchical model-agnostic XCV method for images, named \emph{ShapBPT}, that integrates an adaptive multi-scale partitioning algorithm with the Owen approximation of the Shapley coefficients.
    We repurpose the BPT (Binary Partition Tree) algorithm~\cite{salembier2000BPT} to effectively construct hierarchical structures for explainability. 
    This approach overcomes the limitations of the inflexible hierarchies of state-of-the-art methods such as SHAP.
    \item An empirical assessment of the proposed method on natural color images showcasing its efficacy across various scoring targets, in comparison to established state-of-the-art XCV methods, and a controlled human-subject study comparing explanation interpretability across methods.
\end{enumerate}

\vspace{3pt}
\noindent
We show that the proposed approach surpasses existing Shapley-based model-agnostic XCV methods that do not leverage on data-awareness, and at the same time it achieves a significantly faster convergence rate. 
This efficiency stems from the fact that, on average, fewer recursive applications of the Owen formula (i.e. expansions of the partition hierarchy) are needed to accurately localize objects when using a \textit{data-aware} partition hierarchy, such as the proposed BPT hierarchy, compared to other hierarchies.
As far as we know, this is the first XCV method that combines the Owen formula with a data-aware partition hierarchy for image data, and with this paper we prove the effectiveness of this combined strategy for interpreting ML classifiers.


\section{Methodology}
\label{sec:Methodology}
A fundamental ML objective is to discover a function $f: \mathcal{X} \rightarrow \mathcal{Y}$ that effectively approximates a response $y \in \mathcal{Y}$ corresponding to a given input $x \in \mathcal{X}$. 
For the sake of simplicity, we assume $\mathcal{Y} \subseteq \mathbb{R}$ and $\mathcal{X} \subseteq \mathbb{R}^n$.
In many practical cases, only some components of $x$ significantly influence the response $y = f(x)$.
Understanding the relative importance, or \emph{contribution}, of each component $x_i$ of $x$ in determining the value of $y$ by $f$ is a central problem in XCV.
An important approach~\cite{covert2021explaining} for assessing these contributions is through \emph{feature removal} (also called \emph{masking}), where certain values of $x$ are replaced with values from a specified context-dependent background set. 
Let $\nu_{f,x} : 2^{|\mathcal{X}|} \rightarrow \mathcal{Y}$ be a \emph{masking function} for $f(x)$, where $\nu_{f,x}(S)$ represents the evaluation of the resulting model when only the elements in the subset $S$ of $x$ are retained, while the others are masked. 
In the following, we will denote $\nu_{f,x}$ as $\nu$.

\paragraph{Shapley values.}
We consider the setup of a $n$-coalition game $(\setN, \nu)$, which is analogous to an importance scores attribution task in XCV\,\cite{ijcai2022ShapleyInML}.
The finite set $\setN = \{1, \ldots, n\}$ is the set of players (\emph{features}). Each nonempty subset $S \subseteq \setN$ is a \emph{coalition}, and $\setN$ is itself the \emph{grand coalition}.
A \emph{characteristic function} $\nu : 2^n \rightarrow \mathbb{R}$ assigns to each coalition $S$ a (real) \emph{worth value} $\nu(S)$, and it is assumed that $\nu(\varnothing) = 0$ (it is always possible to ensure $\nu(\varnothing) = 0$ by translation of the equation system).
A \emph{marginal contribution} of a player $i$ to a coalition $S$ (assuming $i \not\in S$) is given by
\begin{equation} \label{eq:marginal}
    \Delta_i(S) = \nu(S \cup \{i\}) - \nu(S)
\end{equation}
Semivalues\,\cite{dubey1981semivalues}, weighted sums of marginal contributions \eqref{eq:marginal}, were introduced as a method for fairly distributing the total value $\nu(\setN)$ of the grand coalition $\setN$ among its members.
The Shapley value\,\cite{shapley1953value}, a well-known semivalue, demonstrates favorable axiomatic properties and has been used effectively to explain ML models \cite{ijcai2022ShapleyInML}.

\paragraph{Hierarchical coalition structures (HCS).}
A fixed a-priori \emph{coalition structure}~\cite{lopez2009relationship,owen2013book,owen1977values} for the $\setN$ players is a finite set  $\{ T_1, \ldots, T_m \}$ of $m$ partitions of $\setN$ (i.e. $\cup_{k=1}^m T_k = \setN$, and $T_i \cap T_j \neq \varnothing ~\Leftrightarrow~ i=j$). 
Elements $T_i$ are usually called \emph{partitions}, \emph{coalitions}, \emph{teams} or \emph{unions}.

We consider a recursive definition of a hierarchical coalition structure, where each partition $T$ can be either an \emph{indivisible partition} or a \emph{sub-coalition structure} itself $T = T_{1} \cup \ldots \cup T_{m}$. 
Let $T{\downarrow}$ be the (downward) recursive partitioning of $T$, defined as
\begin{equation}\label{eq:hierarchyCS}
    T{\downarrow} = \begin{cases}
        \{T_{1}, \ldots, T_{m}\} & \text{if $T$ admits sub-coalitions} \\
        \bot & \text{if $T$ is indivisible}
    \end{cases}
\end{equation}
We denote with $\TT$ the HCS root, and assume w.l.o.g. that $\TT$ contains all the elements of $\setN$.

A special case of HCS happens when each sub-coalition structure is made by two partitions, i.e. the hierarchy forms a binary tree. We refer to these structures as \emph{binary hierarchical coalition structures} (BHCS).
In that case the recursive downward partitioning of $T$ can be simplified as
\begin{equation}\label{eq:hierarchyCSbinary}
    T{\downarrow} = \begin{cases}
        \{T_{1},\, T_{2}\} & \text{\!if $T$ admits a binary sub-coalition} \\
        \bot & \text{\!if $T$ is indivisible}
    \end{cases}
\end{equation}

\paragraph{The Owen approximation for Binary HCS.}
Computing exact Shapley values is at least \#P-hard\,\cite{deng1994complexity}, which is unfeasible for image data with hundreds or thousands of features (pixels). 
An approximate approach, introduced by \cite{owen1977values}, can be used to drastically reduce the cost by grouping features into hierarchical coalitions.
This concept has been pioneered for images by the SHAP Partition Explainer\,\cite{shapPartitionExplainer,shrikumar2017DeepLIFT,lundberg2017unified}.

A \emph{coalition value} $\Omega_i(\TT)$ represents the worth of the player $i$ in a game with coalition structure $\TT$, and is known as the Owen coalition value\,\cite{owen1977values}. 
Computing coalition values over a binary HCS $T$ as defined in \eqref{eq:hierarchyCSbinary} can be done by recursively composing a coalition $Q$ using the formula

\begin{equation}\label{eq:BPT}
    \Omega_i(Q, T) = 
    \begin{cases}
        \displaystyle
        \frac{1}{2} \Omega_i(Q \cup T_{2},\, T_{1}) \,+\,
        \frac{1}{2} \Omega_i(Q,\, T_{1}) & 
            \text{if}~T{\downarrow} = \{T_{1},\, T_{2}\}
        \\[4pt]
        \frac{1}{|T|} \Delta_{T}(Q)
       & \text{if $T$ is indivisible}
        \\
    \end{cases}
\end{equation}
with $\Omega_i(\TT) = \Omega_i(\varnothing, \TT)$.
The former case of Eq.\,\eqref{eq:BPT} deals with coalitions $T$ that admit a sub-coalition structure $T{\downarrow} \neq \bot$.
We assume, for notational simplicity and without loss of generality, that $i \in T_{1}$.
The latter case of Eq.\,\eqref{eq:BPT} deals with indivisible coalitions. In that case, the formula computes a marginal contribution (uniformly divided) of all players of $T$ w.r.t. the coalition $Q$ formed recursively.

In the rest of the paper, we will refer to the Owen approximation of the Shapley values simply as the Shapley values. 
Note that Eq.\,\eqref{eq:BPT} is not found in published literature (as far as we know), and its complete derivation is therefore provided in the Technical Appendix.

\begin{theorem}
    \label{th:cost}
    \textbf{Computational cost.}
    Consider a BHCS consisting of a balanced tree of depth $d$.
    The time complexity of Eq.\,\eqref{eq:BPT} is in the order of $O(4^d)$ evaluations of the $\nu$ function.
\end{theorem}
\begin{proof}
    Derivation is in Technical Appendix.
\end{proof}

Theorem\,\ref{th:cost} highlights the exponential cost of Eq.\,\eqref{eq:BPT}.
However, practical implementation of Eq.\,\eqref{eq:BPT} 
do not rely on expanding a fully balanced BHCS tree to a fixed depth $d$.
Instead, they employ an adaptive splitting strategy that is not limited to balanced trees.
In this adaptive case, a total budget $b$ of evaluations of the masked model $\nu$ is allocated. 
The adaptive algorithm then iteratively explores the tree hierarchy, at each iteration splitting the partition $T$ that maximizes the sum of its Shapley values, $\sum_{i \in T} \Omega_i(\varnothing, \TT)$. Each partition split requires $2$ model evaluations. 
A pseudo-code of this adaptive algorithm is provided in the Technical Appendix.
Despite adaptively ignoring certain coalitions, the cost of exploring the hierarchy at depth $d$ remains exponential, as stated in Theorem\,\ref{th:cost}.

\section{Hierarchical Coalition Structures for Images}
\label{sec:HCS_image_data}


Calculating Owen coalition values for image data necessitates a well-defined hierarchical structure that captures both spatial relationships and image semantics. 
Our approach is aimed at addressing limitations in existing methods, by emphasizing the importance of these factors in coalition formation.
We therefore consider and compare both \emph{data-agnostic} and \emph{data-aware} approaches.

In a \emph{data-agnostic} approach, partitions are created based on simple geometric divisions, like grids or quadrants. 
The \emph{Axis Aligned grid hierarchy} (AA hereafter)  is one such approach to building hierarchical coalition structures, adopted by the SHAP's Partition Explainer\,\cite{shapPartitionExplainer} and by h-SHAP\,\cite{teneggi2022hShap}.
In an AA hierarchy, each partition $T$ corresponds to a rectangular region within the image,
and $T{\downarrow}$ splits the rectangular region of $T$ in half along the longest axis.
This splitting process continues until indivisible (unitary) regions (i.e. single pixels) are reached, or an evaluation budget $b$ is consumed.
The main limitation of this approach is that properly localizing the relevant regions within an image may require a large number of recursive evaluation of the Owen's formula \eqref{eq:BPT}, and this evaluation follows the $O(4^d)$ time cost of Theorem \ref{th:cost}.

\begin{figure}[H]
    \centering
    \includegraphics[width=0.6\linewidth]{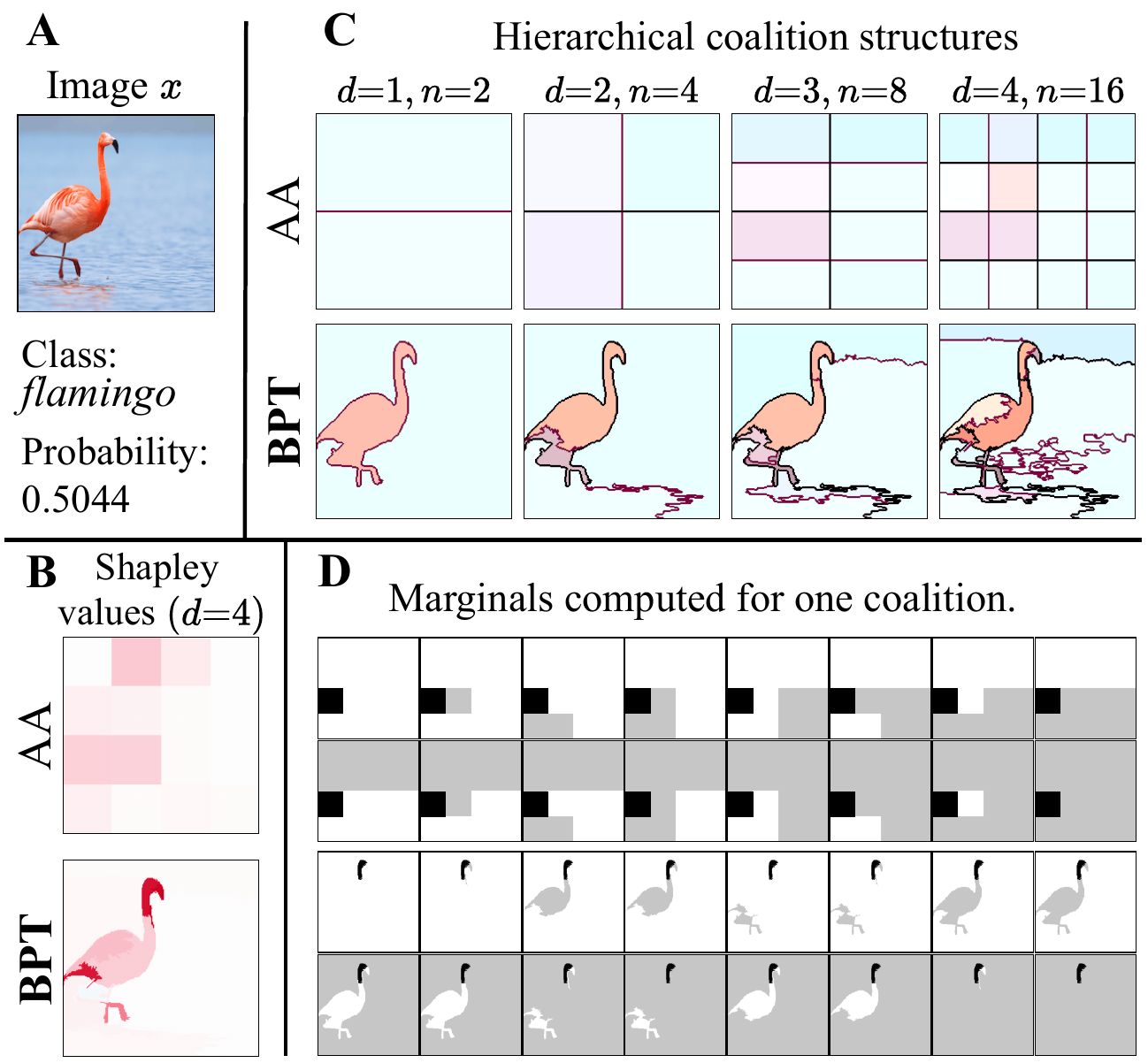}
    \caption{AA and BPT coalition structures for a sample image classification using a ResNet50 model.}
    \label{fig:HCS}
\end{figure}

In a basic \emph{data-aware} approach, morphological features within the image guide the partitioning process. This approach, pioneered by \cite{ribeiro2016lime} with LIME, utilizes a pre-defined segmentation algorithm to divide the image into regions (patches). 
Although effective, the main limitation is the lack of an effective feedback loop within the explanation method. If the segmentation is inaccurate, the resulting explanation is poor, and there is no opportunity for refinement.

\begin{figure}[H]
    \centering
    \includegraphics[width=1.0\linewidth]{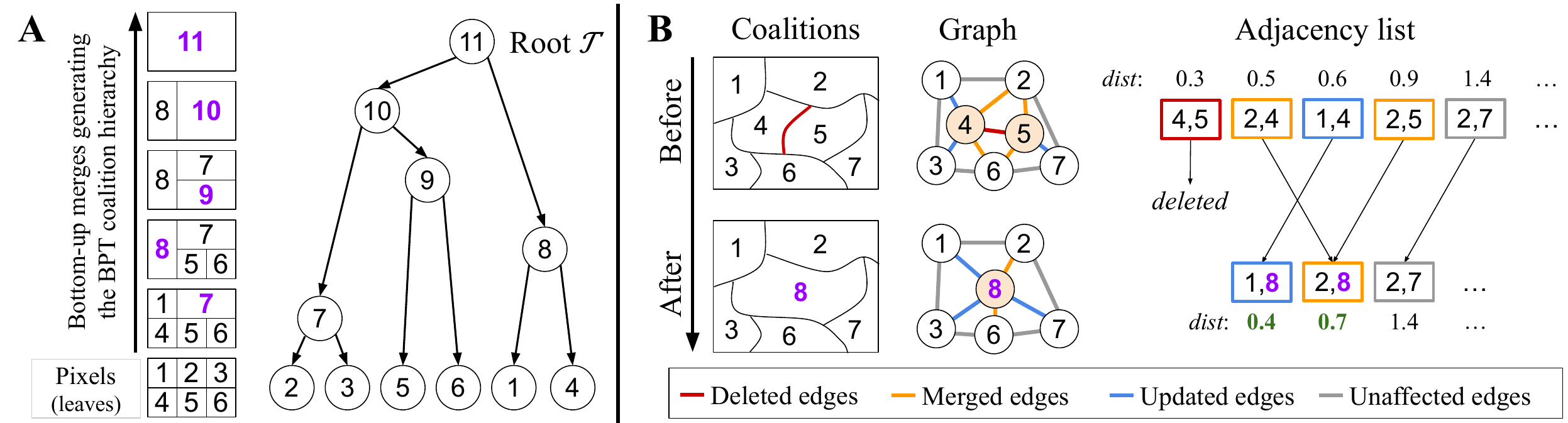}
    \caption{\textbf{(A)} BPT generating by bottom-up merging coalitions from the pixels (1--6) to the root (11). \textbf{(B)} Details of one merging step $T_8{\downarrow}=\{T_4, T_5\}$ on some arbitrary coalition structure.}
    \label{fig:BPT}
\end{figure}

A notable algorithm for hierarchical segmentation, that fits well with Eq.\,\eqref{eq:BPT}, is the \emph{Binary Partition Tree} (BPT)\,\cite{randrianasoa2018multiFeatureBPT}, originally developed for multiscale image representation in MPEG-7 encoding\,\cite{salembier2000BPT}. 
The intuitive principle is that portions of an image with similar color and coherent shape are highly likely to have similar Shapley values, thereby maximizing the effectiveness of  Eq.\,\eqref{eq:BPT}. 

Theorem\,\ref{th:cost} shows that the Owen approximation cost increases rapidly if a large number of coalitions need to be evaluated recursively. \
Therefore, an effective BHCS must satisfy these requirements:
\begin{enumerate}
    \item[R1] As few recursive cuts as possible to reach the relevant regions, as each cut increases the required evaluation budget $b$ exponentially;
    \item[R2] Partitions should not be fixed, since the relevant regions are not known in advance.
\end{enumerate}
AA hierarchies do not satisfy R1, and  most a-priori segmentation algorithms do not satisfy R2.
The solution that we propose, which constitutes the main contribution of this paper, is a novel hybrid method that finally satisfies the two aforementioned requirements by combining a refinable a-priori hierarchical coalition structure (the BPT) aligned with the morphological features of the image (e.g., color uniformity, pixel locality) together with an a-posteriori splitting strategy based on the distribution of Shapley values (as in the Partition Explainer).
This combination results in significantly fewer recursive applications of the Owen formula needed to accurately localize objects, compared to data-agnostic coalition structures.
As we shall see in the experimental section, this approach usually gets a faster convergence than other Shapley-based methods, paired with accurate shape recognition of the classified objects.

\begin{example}
Figure~\ref{fig:HCS} presents a sample image (A) with its Shapley explanations (B), computed using Eq.\,\eqref{eq:BPT} on AA and BPT hierarchical coalition structures (C) up to depth $d=4$. The first four tree hierarchy levels in (C) highlight the data-aware nature of BPT. Each coalition value is derived from a weighted sum of eight marginals $\widehat{\varphi}_i(Q, T)$, with the highest-value marginals shown in (D), where $Q$ and $T$ correspond to the grey and black regions. 
\end{example}

\paragraph{Generating BPT hierarchies.}
A \emph{BPT hierarchy} captures how we can progressively merge\,\cite{randrianasoa2018multiFeatureBPT} the $n$ pixels of an image $x$ into larger regions, forming a quasi-balanced binary tree.
Tree construction is bottom-up, starting from an initial coalition structure $\TT_{[1]} = \bigl\{ T_{1}{=}\{1\}, \ldots, T_{n}{=}\{n\} \bigr\}$ made by $n$ unitary and indivisible partitions, where the features $1, \ldots, n$ represents the individual pixels of the image.
Two partitions $T_i,T_j \in \TT_{[k]}$ are \emph{adjacent} if there is at least one pixel of $T_i$ that is adjacent to a pixel of $T_j$ in the image.
The BPT construction involves merging adjacent partitions iteratively.

A \emph{coalition merge} of $\TT_{[k]}$ is a new coalition structure $\TT_{[k+1]}$ where two adjacent partitions $T_i, T_j \in \TT_{[k]}$ are removed and replaced by a new partition $T_{n+k}$, s.t. $T_{n+k} = T_i \cup T_j$ and $T_{n+k}{\downarrow} = \{T_i,T_j\}$.

The two adjacent partitions $T_i, T_j$ of $\TT_{[k]}$ to be merged are selected by minimizing a \emph{data-aware} distance function.
Prior work\,\cite{randrianasoa2018multiFeatureBPT,randrianasoa2021agatBPT} on BPTs shows that color range $\times$ perimeter scores correlate with perceptual region uniformity, and area helps in keeping the tree hierarchy balanced.
With this knowledge we define
\begin{equation}
\label{eq:bpt:distance}
    \mathit{dist}(T_i,T_j) =  
        \mathit{clr}^2(T_i,T_j) \cdot \mathit{area}(T_i,T_j) \cdot \sqrt{\mathit{pr}(T_i,T_j)}
\end{equation}
as a distance criteria, where $\mathit{clr}^2(T_i,T_j)$ is the sum of the squared color ranges of $T_i \cup T_j$, for all color channels, 
and $\mathit{area}(T_i,T_j)$ and $\mathit{pr}(T_i,T_j)$ are the area and the perimeter of $T_i \cup T_j$, respectively.
A sensitivity ablation analysis that supports the rationale of Eq.\,\eqref{eq:bpt:distance} is in the Technical Appendix.

A \emph{merging sequence} $\TT_{[1]}\rightarrow\TT_{[2]}\rightarrow\ldots\rightarrow\TT_{[n]}$ is a sequence of $n-1$ coalition merges.
The sequence ends with the coalition structure $\TT_{[n]} = \bigl\{ T_{2n-1} \bigr\}$, having a single partition with all pixels.
At this point, all non-unitary partitions $T$ at any point in the merging sequence admit a binary sub-coalition structure $T{\downarrow}$.
Therefore, the BPT $\TT_{[n]}$ satisfies Eq.\,\ref{eq:hierarchyCSbinary}, and may become the root $\TT$ of the BHCS. 
An illustration of the algorithm generating the BPT merging sequence is shown in Figure\,\ref{fig:BPT}/A, where the unitary partitions are merged, one by one, until all pixels are merged into the root $\TT$.
The operations needed to perform a single merging step are illustrated in Figure\,\ref{fig:BPT}/B, and a detailed pseudo-code of the BPT algorithm is provided in the Technical Appendix.


\section{Experimental Assessment} 
\label{sec:experiments}

We present a comparative analysis of the performance of the proposed Shapley method using BPT partitions, alongside other state-of-the-art image explainers.

\paragraph{Comparison scores.}
To ensure a robust and comprehensive quantitative evaluation, we consider two score categories: \textit{response-based} and \textit{ground-truth-based}.
The \textit{response-based} score that we consider are the \emph{area-under-curve} (AUC) from \cite{riseScores}, which measure how well the ranked explanation coefficients align with the black-box model's output. These scores do not rely on any predefined notion of “correct” explanation and instead evaluate the internal consistency of the explanation with respect to the model’s own behavior.
Let $S^{[q]} \subseteq \setN$ be the subset of the first $q$-th quantile of elements from $\setN$ with the largest Shapley values.
Define
\begin{equation}
\label{eq:AUC}
\begin{aligned}
\mathit{AUC}^{+} &\!=\!\! 
    \int_{0}^{1}\!\! \nu\bigl( S^{[q]} \bigr) \,\mathrm{d}q 
\\
\mathit{AUC}^{-} &\!=\!\!
    \int_{0}^{1}\!\! \nu\bigl( \setN \setminus S^{[q]} \bigr) \,\mathrm{d}q 
\end{aligned}
\end{equation}
With this definition $\mathit{AUC}^{+}$ (resp. $\mathit{AUC}^{-}$) evaluate the model's behavior as features are progressively included from an empty set (resp. excluded from the full set).
Since we deal with regression models, we rescale\,\cite{hama2023deletion} all $\nu$ values in the $[0, 1]$ range, s.t. all evaluated samples weight uniformly.

The \textit{ground-truth-based} score we consider is the Intersection over Union (IoU) score, which compares the predicted important features with a known \emph{ground truth} subset $G \subseteq \setN$. Ideally $G$ is a set for which $\nu(G) = \nu(\setN)$. This setup is relevant in the context of the \emph{Visual Recognition Challenge} (VRC) \cite{ILSVRC15}, where annotations provide an external reference for which image regions are expected to contribute to classification.
An explanation is a \emph{perfect match} if there is a threshold $q$ for which $S^{[q]} = G$.
Consider the standard \emph{Intersection-over-Union} score $\mathit{J}(A,B) = \frac{\,|A \cap B|\,}{|A \cup B|}$ and define
\begin{equation}
\begin{aligned}
    \AUIoU &= \int_0^1 \mathit{J}(S^{[q]}, G) \,\mathrm{d}q \\
    \maxIoU &= \underset{q \in [0,1]}{\mathrm{max}} 
    \mathit{J}(S^{[q]}, G)
\end{aligned}
\label{eq:IoU}
\end{equation}
The score $\AUIoU$\,\cite{GANGOPADHYAY202364} is the area under the IoU curve, defined by the IoU values in the range $q \in [0,1]$, and $\maxIoU$ is the curve maximum. The $\AUIoU$ is maximal when the explanation perfectly matches the ground truth mask, and in such case $\maxIoU = 1$.

\begin{figure}[H]
    \centering
    \includegraphics[width=0.5\linewidth]{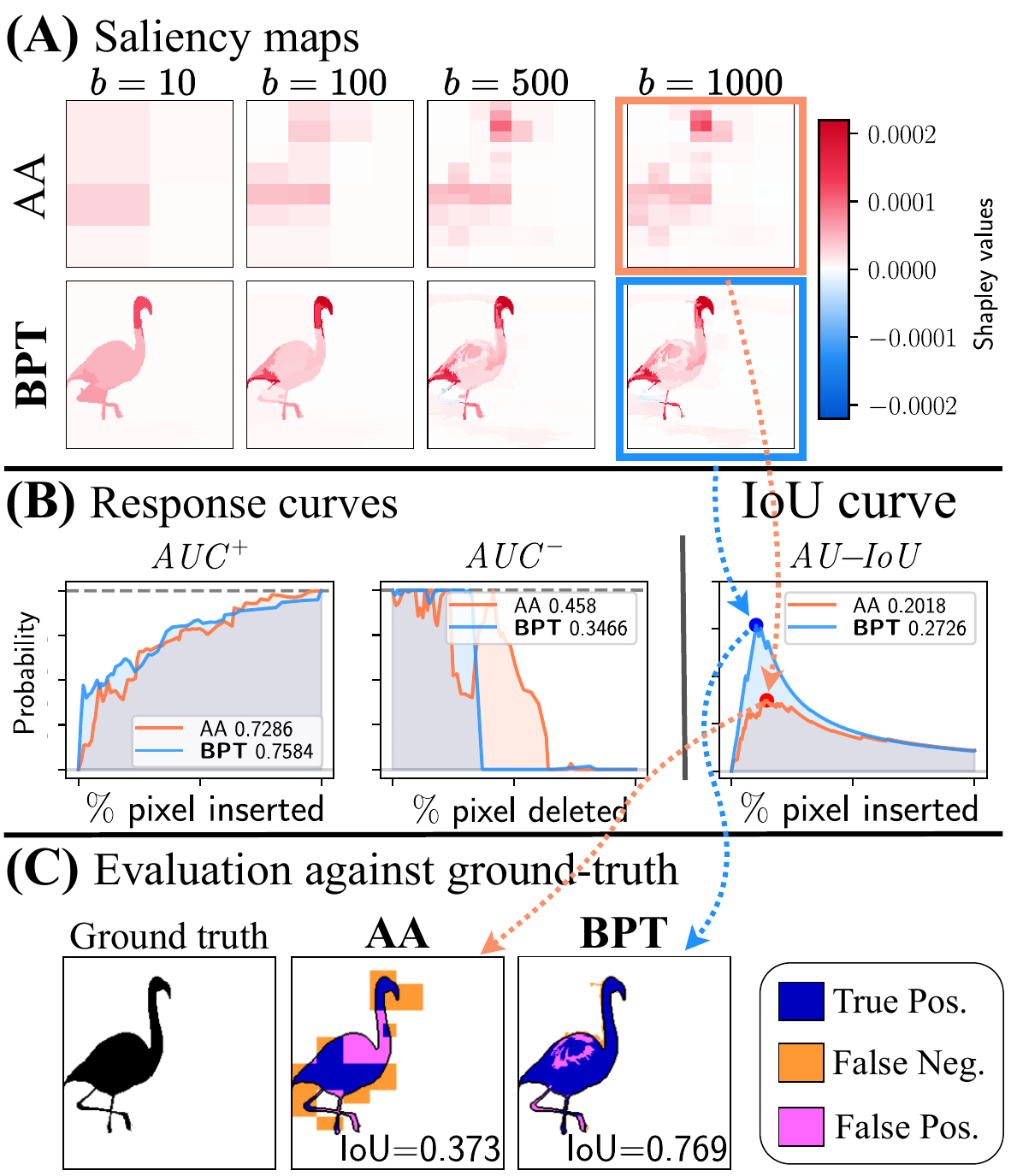}
    \caption{Shapley values for AA and BPT coalition structures, for different values of the budget $b$.}
    \label{fig:shapley}
\end{figure}

\begin{example}
Figure~\ref{fig:shapley} shows the Shapley values computed using Eq.\,\eqref{eq:BPT} on the AA and BPT coalition structures, 
by refining the most significant coalition using a budget $b$ of model evaluations \textbf{(A)}, 
with $b$ equal to $10$, $100$, $500$ and $1000$ samples, respectively.
The area identified by the threshold $q$ obtaining the maximal IoU is depicted in \textbf{(C)}.
The plots \textbf{(B)} depict the response curves for the AUC scores \eqref{eq:AUC} and \eqref{eq:IoU}, for the case $b{=}1000$.
In the example, BPT demonstrates its improved object region recognition w.r.t. AA.
\end{example}

\begin{figure*}[t]
    \centering
    \includegraphics[width=1.0\linewidth]{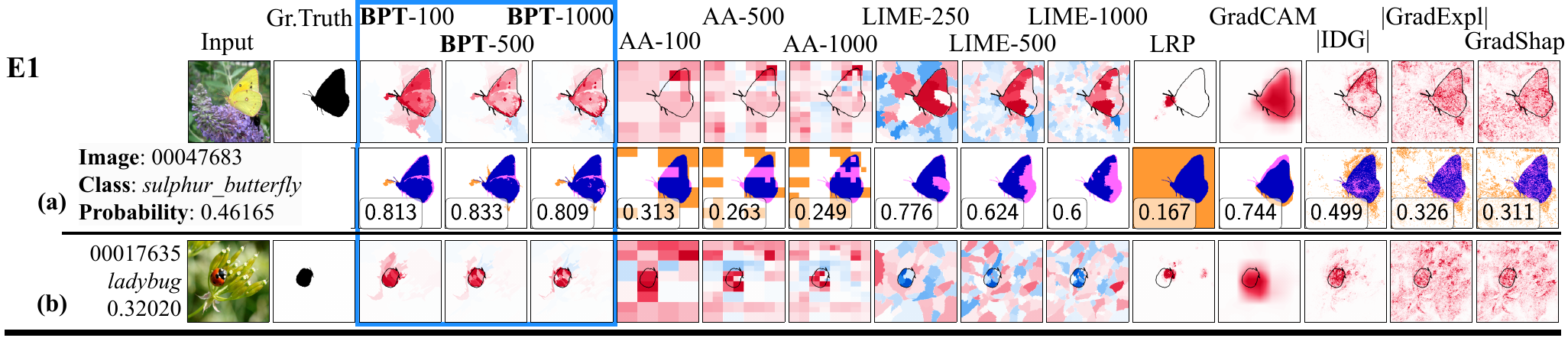}
    \includegraphics[width=1.0\linewidth]{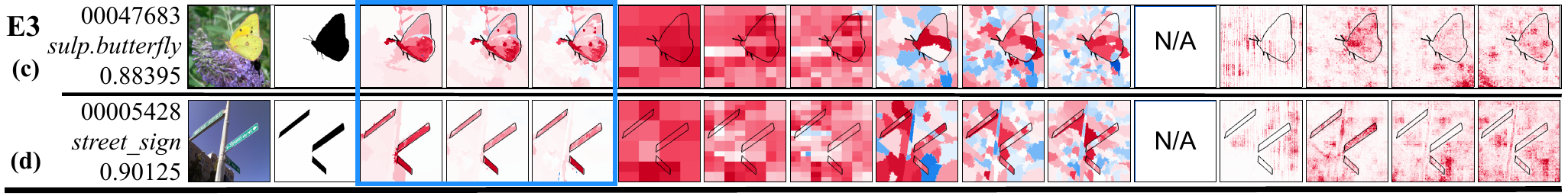}
    \includegraphics[width=1.0\linewidth]{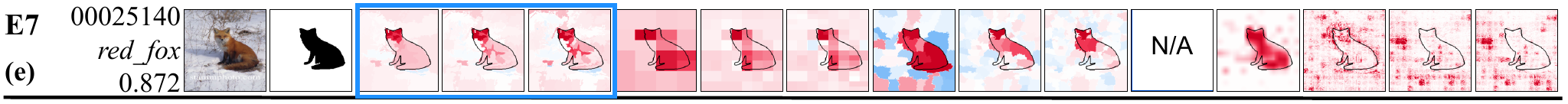}
    \includegraphics[width=1.0\linewidth]{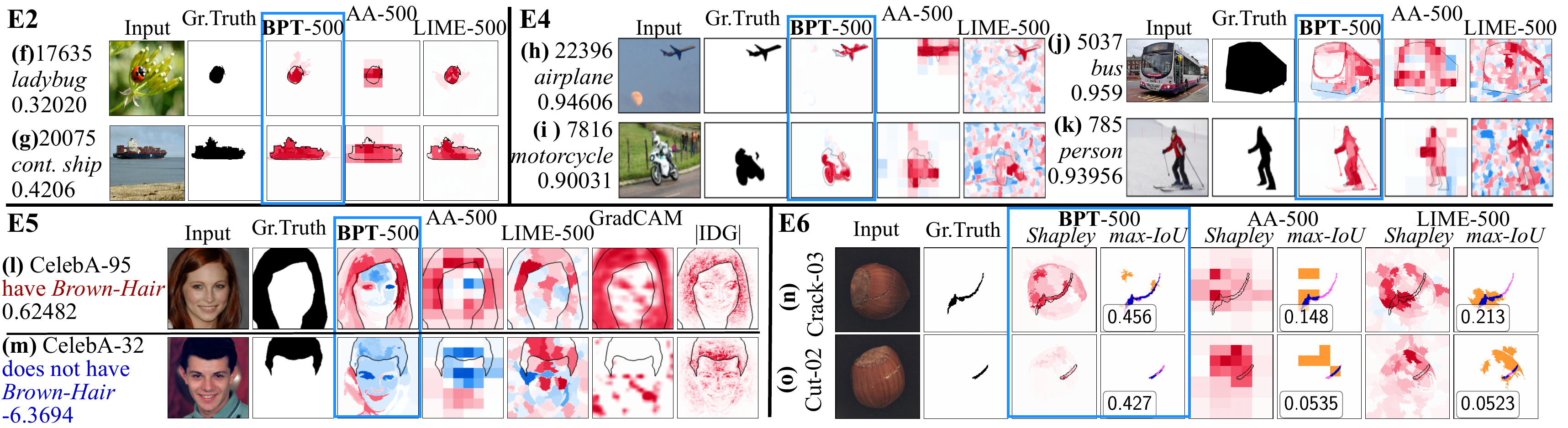}

    \caption{Selected saliency maps from experiments \textbf{E1}--\textbf{E7} (summarized in Table \ref{tab:summary}) for various computer vision ML tasks.}
    \label{fig:saliencyMaps}
\end{figure*}

\paragraph{Compared methods.}
We run a comparative analysis using several state-of-the-art XCV methods, categorized into two groups.
The first group comprises Shapley-based methods, chosen for their compatibility with our proposed approach. 
They include:
\textbf{BPT}-$b$: our proposed Shapley explanation method with BPT partitions, with sample budgets $b$ of 100, 500, and 1000 samples;
\textbf{AA}-$b$: the SHAP Partition Explainer\,\cite{shapPartitionExplainer}, utilizing Axis-Aligned partitions with $b$ of 100, 500, and 1000 samples;
\textbf{LIME}-$b$: LIME\footnote{Although LIME does not generate Shapley values, it has theoretical and practical similarities to them\,\cite{lundberg2017unified}.} explanation\,\cite{ribeiro2016lime} with budget $b$ and with $b/5$ segments,  with $b$ being 100, 500, and 1000.

The second group consists of gradient-based methods, included in our analysis due to their widespread usage. They include:
\textbf{GradExpl}: the Gradient Explainer from the SHAP package\,\cite{lundberg2017unified}, using the default of 20 samples;
\textbf{GradCAM}: the Gradient-weighted Class Activation Mapping introduced by\,\cite{gradcampytorch};
\textbf{IDG}: the Integrated Decision Gradient method of\,\cite{walker2024integrated};
\textbf{LRP}: Layer-wise Relevance Propagation of \cite{bach2015pixel,ancona2018towards} from Captum;
\textbf{GradShap}: gradient Shap\,\cite{sundararajan2017axiomatic}.
For \emph{GradExpl} and \emph{IDG}, we utilize the absolute values of the produced explanations, resulting in superior scores compared to the signed values.

\begin{table}[t]
    \centering\scalebox{1.0}{
    \begin{tabular}{c|c|c|c|c}
        \hline
        \textbf{} & \textbf{Dataset} & \textbf{Size} & \textbf{Model} & \textbf{Short description} \\
        \hline
        \hyperref[app:Experiment-E1]{E1} & ImageNet-S$_{50}$ & 574 & ResNet50 & Common ImageNet setup \\
        \hyperref[app:additionalExperiment-E2]{E2}& ImageNet-S$_{50}$ & 574 & Ideal & Linear ideal model \\
        \hyperref[app:additionalExperiment-E3]{E3} & ImageNet-S$_{50}$ & 621 & SwinViT & Vision Transformer\\
        \hyperref[app:additionalExperiment-E4]{E4} & MS-COCO & 274 & Yolo11s & Object detection \\
        \hyperref[app:additionalExperiment-E5]{E5} & CelebA & 400 & CNN & Facial attrib. localization \\
        \hyperref[app:additionalExperiment-E6]{E6} & MVTec & 280 & \!VAE-GAN\! & Anomaly Detection \\
        \hyperref[app:additionalExperiment-E7]{E7} & ImageNet-S$_{50}$ & 593 & ViT-Base16 & Vision Transformer \\
        \hline
        \hyperref[app:experiment-E8]{E8} & \multicolumn{4}{c}{User preference study using E1 saliency maps.} \\
        \hline
    \end{tabular}
    }
    \caption{Summary of the experiments.}
    \label{tab:summary}
\end{table}

\paragraph{Experiments.}
ShapBPT has been tested extensively over multiple computer vision tasks, models and datasets. 
Table~\ref{tab:summary} reports a summary of the experiments. 
Figure~\ref{fig:saliencyMaps} depicts examples of the generated saliency maps for the first seven experiments, which helps to get a first intuition of the characteristics of the BPT method. 
Each row reports the image, the ground truth $G$, and the saliency maps.
We show all the fourteen tested method only for \textbf{E1}. 
The boundaries of $G$ are drawn overlapped to the saliency maps.
To illustrate the evaluation process, for the first image, we also report the optimal IoU w.r.t.~$G$.
While all the tested methods seem to somewhat agree on the recognition area, the practical behavior of BPT seems in line with its theoretical assumption that splitting the image partitions following the morphological image boundaries leads to better object recognition, and better separation from the background.

Experiments are briefly detailed in the following. Further details, examples and results are in the Technical Appendix.

\begin{figure}[t]
    \centering
    \includegraphics[width=1.0\linewidth]{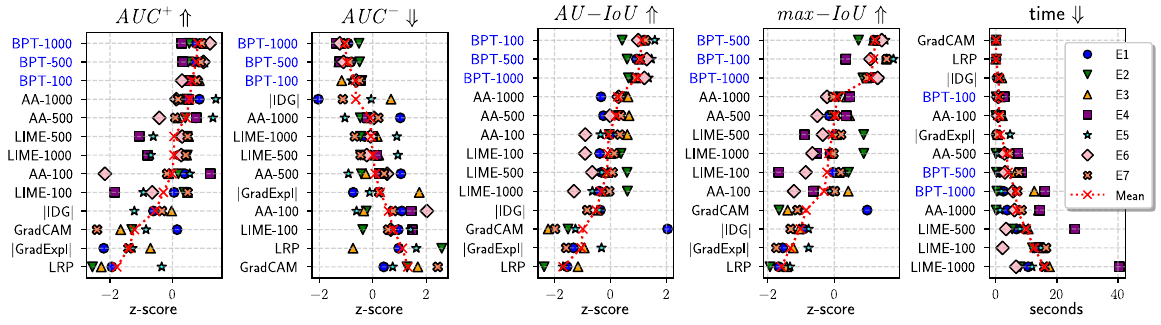}
    \caption{Results for all scores across the \textbf{E1}--\textbf{E7} experiments, with methods (on Y axis) ranked by performance (top to bottom).}
    \label{fig:resultsAll}
\end{figure}

\medskip\noindent Experiment \textbf{E1} 
uses the \emph{1K-V2} pretrained\,\cite{vryniotis2021train} ResNet50\,\cite{resnet50} model (from PyTorch, accuracy $80.858\%$) over the ImageNet-S$_{50}$ dataset\,\cite{gao2022lussImageNetS} with ground truth masks (574 images in total).
Replacement value is uniform gray.
In general BPT explanations (columns 3--5) show a better tendency of identifying the partition borders, cutting the recognized object from the background.
In that sense, they share similarities with the explanations of LIME, but without the typical LIME noise, and without relying on a fixed, inflexible segmentation.
Moreover BPT explanations look a lot more in accordance to those of GradCAM, but without the blurriness that the latter adds.

\medskip\noindent Experiment \textbf{E2} evaluates an ideal linear model which perfectly follows the ground truth. Let
$\nu_\text{lin}(S) = \tfrac{|S \cap G|}{|G|}$
be an ideal linear predictor that outputs the proportion of pixels of $S$ that belong to the ground truth $G$. Since $\nu_\text{lin}$ is not a neural network, CAM methods cannot be used and are excluded.

\medskip\noindent Experiment \textbf{E3} uses the pretrained Vision Transformer model \emph{SwinViT}\,\cite{liu2021swin} from pytorch (acc. $81.4$\%).
It is interesting to see that all methods except BPT produce significantly more confused saliency maps, attributing a lot of importance to background features and with little focus to the actual classified objects. On the contrary, saliency maps obtained by the BPT method are clear and focused.

\medskip\noindent Experiment \textbf{E4} uses the Ultralytics Yolo11s model\,\cite{yolo11ultralytics} pre-trained for the MS-COCO dataset\,\cite{MSCOCO} which has diverse image sizes and a wider range of details than ImageNet. 
The XCV task involves highlighting detected objects.

\medskip\noindent Experiment \textbf{E5} uses a pre-trained CNN model\,\cite{Kartik} to predict the presence or absence and the localization of facial features like \emph{brown hair} and \emph{eyeglasses} on the CelebA-HQ dataset\,\cite{karras2018progressive}. 
The XCV task focuses on localizing regions positively (red) or negatively (blue) influencing predictions. For this kind of tasks, Shapley values correctly distinguish positive and negative contributions, unlike CAM methods.

\medskip\noindent Experiment \textbf{E6} 
focuses on explaining an Anomaly Detection (AD) system. It builds on the \cite{ravi2021general} methodology, and uses a convolutional VAE-GAN model for anomaly localization. 
The MVTec dataset\,\cite{bergmann2019mvtec} (\emph{hazelnut} category) is used, with 280 high-quality images of defective (with ground truth masks) and non-defective objects.
The anomaly map captures reconstruction errors, reflecting both anomalies and noise, and the XCV tasks consists in separating true anomalies from noise.

\medskip\noindent Experiment \textbf{E7} Similar to E1 using the ViT-Base 16 model\,\cite{dosovitskiyimage}.

\paragraph{Numerical results.}
Figure~\ref{fig:resultsAll} summarizes the results from all experiments, with a separate table for each of the four scores, plus one for the wall-clock evaluation time\footnote{Using: Intel Core i9 CPU, Nvidia 4070 GPU, 16GB RAM.}.
To ensure fairness across experiments, scores have been standardized accordingly.
A red line indicates the overall mean across all experiments.
Methods are ordered based on their mean scores from top (better) to bottom (worse).
To assess statistical significance, we conducted one-way ANOVA tests for each score, testing the null hypothesis ($H_0$) of equal means across all sample populations, with $p$-value significance threshold of $0.05$. In all cases, the null hypothesis was rejected, indicating that the results are statistically significant. Results are reported in the Technical Appendix.

From Figure~\ref{fig:resultsAll}, we observe that BPT consistently outperforms AA and other compared methods across various models, scoring methods, and datasets.
This supports the intuition that tailoring partitions to the characteristics of the data is beneficial.
Moreover, BPT maintains its advantage even under resource constraints, as BPT-100 already surpasses most competing methods. 
This highlights its adaptability to different experimental setups and its robust generalization ability across datasets and models. 
In particular, ShapBPT seems to work well also with Vision Transformers (\textbf{E3} and \textbf{E7}), which are known for their robustness to partial object occlusion\,\cite{Englebert_2023_ICCV}.

\paragraph{E8 - User preference study.}
In addition to the automated metrics, we measured \emph{perceived usefulness} with a small controlled user study with 20 participants. 
Each subject viewed four randomly selected images from \textbf{E1} and ranked the four explanation maps (BPT-1000, GradCAM, LIME-1000, AA-1000) from most to least helpful for understanding the model’s prediction, yielding 80 rank-lists per method.
A Friedman test determined the significance 
($\chi^2_{(k=3)} {=} 19.56$, $p$-value${=}0.0002$, $H_0$ rejected).
BPT emerged as the winner, ranked first in 51\% of the cases with an average rank of 1.79, trailed by GradCAM (33\% first-place, mean 2.41) and LIME (10\%, mean 2.56), while AA was seldom preferred (6\%, mean 3.24). 
Human preference seems therefore to partially confirm the quantitative metrics of the \textbf{E1}--\textbf{E7} experiments.
Full details are in the Technical Appendix.

\section{Discussion and Other Related Works}
\label{sec:discussion}

Our evaluation combines both \textit{ground-truth-based} metrics (IoU) and \textit{response-based} metrics (AUC) to provide a comprehensive and reliable assessment of the methods. 
IoU scores are included because they are the standard evaluation metric in object detection benchmarks~\cite{rezatofighi2019generalized}. While deep learning models may show misalignments between the ground truth $G$ and the model's learned representation, this should not introduce bias in the $\AUIoU$ and $\maxIoU$ scores, as all methods are evaluated under the same conditions.
Moreover, experiment \textbf{E2} is fully unbiased, since the ideal model $\nu_\text{lin}$ is a linear model.

A convergence analysis comparing BPT and AA across varying evaluation budgets is in the Technical Appendix.

For LIME, we generated fixed a priori partitions using the \emph{quickshift} algorithm and also tested the more recent \emph{SegmentAnything} (SAM) method, which improves upon \emph{quickshift} but is significantly slower. However, neither of these methods can build the hierarchy of Shapley values adaptively. The limitation of relying on rigid, pre-defined partitions persists, an issue that is addressed by the proposed BPT approach (as outlined in requirement R2). 
It would be interesting to integrate SAM directly into ShapBPT and compare it against BPT. However SAM does not generate a regular HCS\,\cite{knab2025beyond}, which is a key requirement of the Owen formula. 
Constructing a SAM-compatible HCS therefore demands new algorithmic machinery, beyond the scope of the present study, and merits a dedicated investigation as future work. 
We outline the key details in the Technical Appendix.
A discussion on h-Shap limits is in Technical Appendix.

We considered the \emph{relevance mass and rank accuracy} scores~\cite{CLEVRXAI202214} but eventually excluded them, as their reliance on non-negative values does not work well with Shapley values.

User preference results echo the objective ones: a 20-participant study confirmed that the data-aware BPT hierarchy yields explanations humans actually find most useful.


\section{Conclusions}
\label{sec:conclusions}

This paper introduces \emph{ShapBPT}, a model-agnostic explainability method for AI classifiers in computer vision. 
It computes saliency maps by calculating Shapley values using the Owen formula over a data-aware \emph{Binary Partition Tree} (BPT) of the image being explained.
That captures the importance of image features in a way that is both efficient and consistent with Shapley’s axiomatic properties.

Comprehensive cross-dataset benchmarks and a 20-subject preference study consistently place ShapBPT’s data-aware hierarchical partitions ahead of existing XCV explainers, confirming it as a novel, robust method that delivers accurate, budget-efficient, and human-preferred explanations.

\section{Acknowledgments}

This work has received funding from the European Union’s
Horizon research and innovation program Chips JU under Grant Agreement No. 101139769,
DistriMuSe project (HORIZON-KDT-JU-2023-2-RIA). 
The JU receives support from the
European Union’s Horizon research and innovation
programme and the nations involved in the mentioned
projects. The work reflects only the authors’ views; 
the European Commission is not responsible for any use that may
be made of the information it contains.

\bibliographystyle{unsrt}  
\bibliography{references}

\input{technical_appendix}

\end{document}

%% file: technical_appendix.tex
\clearpage


\section{ Technical Appendix}
\label{sec:appendix_section}

\etocsettocstyle{\section*{Appendix Contents}}{}
\localtableofcontents

\vspace{1em}

\subsection{Derivation of Equation \eqref{eq:BPT}}
\label{app:BPTequation}

We present a clear formulation of the Owen approximation of Shapley values within a hierarchical coalition structure, as this specific approach appears to be absent from existing published literature.
To ease our formulation, we start from a simple extension of the Shapley formula:
\begin{equation} \label{eq:ShapleyQ}
    \varphi_i(Q,\setN) = \sum_{S \subseteq \setN \setminus\{i\}}
        \frac{1}{n \cdot\binom{n-1}{|S|}} \Delta_i(Q \cup S)
\end{equation}
where $n$ is the cardinality of $\setN$.
Eq.~\eqref{eq:ShapleyQ} assigns a unique distribution of the total worth $\nu(\setN)$ generated by cooperation among players in a coalition game, and is extended by assuming that all coalitions  $S$ are supported by a persistent set of players $Q$. 
The regular Shapley value\,\cite[Eq.12]{shapley1953value} are obtained from \eqref{eq:ShapleyQ} as $\varphi_i(\varnothing,\setN)$. 
The persistent set $Q$ is used for the Owen approximation.

The Owen coalition value\,\cite{owen1977values} is an extension of the Shapley value, and it is a quantity $\Omega_i(\TT)$ that represents the worth of player $i$ in a game with coalition structure $\TT$.
The original formulation for a two-level coalition structure hierarchy\footnote{In a two-level coalition structure hierarchy $\TT$, we have $\TT{\downarrow} = \{T_1 \ldots T_m\}$, and $\forall\,1 \leq i \leq m$: $T_i{\downarrow} = \bot$.} works as follows.
Consider a player $i$ belonging to team $T_j \in \TT{\downarrow}$. Then
\begin{equation} \label{eq:Owen1977}
    \Omega_i(\TT) = 
        \sum_{\substack{H \subset M \\ j \not\in H}}
        \sum_{\substack{S \subset T_j \\ i \not\in S}}
        \frac{1}{m\cdot\binom{m-1}{|H|}}
        \cdot 
        \frac{1}{t_j \cdot \binom{t_j-1}{|S|}}
        \Delta_i(Q_H \cup S)
\end{equation}
where $M=\{1 \ldots m\}$ is the set of structured coalition indices of $\TT$,
$Q_H = \bigcup_{k \in H} T_k$, and $t_j = |T_j|$.

Eq.~\eqref{eq:Owen1977} can be seen as a two-level Shapley value, where inside a team $T_j$ all coalitions are possible, but once a coalition $S \subset T_j$ is formed, only a restricted \emph{all-or-nothing} form of cooperation with the other teams is possible.
It is possible to rewrite \eqref{eq:Owen1977} by explicitly identifying the Shapley value for the subsets $S$ of $T_j$. 
By doing so with \eqref{eq:ShapleyQ} and applying simple algebraic transformations, we get
\begin{equation} \label{eq:OwenRewritten}
    \Omega_i(\TT) = 
        \sum_{\substack{H \subseteq M \setminus \{j\}}}
        \frac{1}{m \cdot\binom{m-1}{|H|}}
        \varphi_i(Q_H, T_j)
\end{equation}
i.e. the Owen coalition value is defined on the basis of the Shapley value (extended as in Eq.\,\eqref{eq:ShapleyQ}), similarly to the approach of the so-called ``\emph{two-steps value}'' formulation of \cite[p.300]{owen2013book}.

\begin{example}
Consider a coalition structure 
$\TT = \bigl\{ \{1,2\}, \{3,4,5\}, \{6\}\bigr\}$. 
The coalition value $\Omega_1(\TT) = \eta_1(\varnothing,\TT)$ is the weighted sums of eight marginals:
\begin{equation}
\everymath={\displaystyle}
\begin{array}{cc}
    \frac{1}{6}\Delta_1(\varnothing) &\quad \frac{1}{6}\Delta_1(\{2\}) \\ \frac{1}{6}\Delta_1(\{3, 4, 5, 6\}) &\quad \frac{1}{6}\Delta_1(\{3, 4, 5, 6, 2\}) \\
    \frac{1}{12}\Delta_1(\{6\}) &\quad \frac{1}{12}\Delta_1(\{6, 2\}) \\ \frac{1}{12}\Delta_1(\{3, 4, 5\}) &\quad \frac{1}{12}\Delta_1(\{3, 4, 5, 2\}) 
\end{array}
\end{equation}
Since player $1$ is in an a-priori coalition with player $2$, the other two teams $\{3,4,5\}$ and $\{6\}$ can only appear as a whole.
As a consequence, the Owen approximation of the Shapley coefficients only observes some coalitions, that preserve the integrity of the teams that are in a separate branch of the tree hierarchy.
\end{example}

Observe that $\Omega_i(\TT) \neq \varphi_i(\varnothing,\setN)$, as only a selected structured subsets of coalitions are formed (see  \cite{lopez2009relationship} for an in-depth analysis of this relation).

The two-level formulation is easily extended to an arbitrary hierarchy of coalitions, and this idea has been pioneered for image data by the SHAP Partition Explainer\,\cite{shapPartitionExplainer,shrikumar2017DeepLIFT,lundberg2017unified}. 
Therefore a hierarchical \emph{Owen coalition value} can be obtained rewriting Eq.\,\eqref{eq:OwenRewritten} on top of other Owen coalition values for a coalition $T$, as long as $T$ is not an indivisible coalition. 
The concept is also briefly sketched in\,\cite[p.87]{owen1977values}, but we rewrite the equation to have a simple recursive formula that is general for $m$-ary and binary hierarchical coalition structures, as in Eqs. \eqref{eq:hierarchyCS} and \eqref{eq:hierarchyCSbinary}, respectively.

\paragraph{Binary and multi-way tree hierarchies (i.e. $m>2$).}
Consider Eq.\,\eqref{eq:OwenRewritten} and replace the summation over the subsets of indices $M$ with a uniform \emph{subset $U$ of the sub-coalition structure of $T{\downarrow}$}, making the marginal contribution of Eq.\,$\eqref{eq:marginal}$ as the base case of the recursion, and adding a persistent set $Q$ as done for Eq.\,\eqref{eq:ShapleyQ}.
\begin{equation}\label{eq:OwenRecursivePart}
    \Omega_i(Q, T) = 
    \begin{cases}
        \displaystyle
        \sum_{\substack{U \subseteq T{\downarrow} \setminus \{T_j\}}}
        \frac{1}{m \cdot\binom{m-1}{|U|}}
        \Omega_i(Q \cup Q_U,\, T_j)
        & \text{if}~T{\downarrow} = \{T_{1} \ldots T_{m}\}
        \\[15pt]
        \frac{1}{|T|} \Delta_{T}(Q)
        & \text{if $T$ is indivisible}
        \\
    \end{cases}
\end{equation}
where $Q_U = \bigcup_{k=1}^{|U|}U_k$, and assuming $T_j$ contains $i$.
As before, indivisible coalitions receive uniform attributions among all players. 
The Owen coalition value for player $i$ using Eq.\,\eqref{eq:OwenRecursivePart} is obtained from $\Omega_i(\varnothing, \TT)$, with $\TT$ the HCS root.
When $\TT = \{\setN\}, \TT{\downarrow} = \bot$, then Eq.\,\eqref{eq:OwenRecursivePart} reduces to $\varphi_i(Q, \setN)$, which is trivially equivalent to Eq.\,\eqref{eq:ShapleyQ}. 
Using a two-level HCS, then Eq.\,\eqref{eq:OwenRecursivePart} is equivalent to Eq.\,\eqref{eq:Owen1977} and Eq.\,\eqref{eq:OwenRewritten}.
For arbitrary nested hierarchies, the equation expands, generating the coalitions $Q$ that may pair with the set $T$ containing player $i$, following the hierarchy constraints.

\begin{example}
Consider a three-level HCS 
$$
\TT = \Bigl\{
  \bigl\{ \{1,2\}, \{3,4\} \bigr\},
  \bigl\{ \{5,6\}, \{7\}, \{8\} \bigr\}
\Bigr\}
$$
The hierarchical coalition value $\Omega_1(\varnothing,\TT)$ is the weighted sums of eight marginals:
\begin{equation}
\everymath={\displaystyle}
\begin{array}{cc}
    \frac{1}{8}\Delta_1(\varnothing) &\quad \frac{1}{8}\Delta_1(\{2\}) \\ \frac{1}{8}\Delta_1(\{5, 6, 7, 8\}) &\quad \frac{1}{8}\Delta_1(\{5, 6, 7, 8, 2\}) \\
    \frac{1}{8}\Delta_1(\{3, 4\}) &\quad \frac{1}{8}\Delta_1(\{3, 4, 2\}) \\ \frac{1}{8}\Delta_1(\{5, 6, 7, 8, 3, 4\}) &\quad \frac{1}{8}\Delta_1(\{5, 6, 7, 8, 3, 4, 2\}) 
\end{array}
\end{equation}
Coalitions can pair with player $1$ following the hierarchy. 
Therefore $\{3,4\}$ and $\{5,6,7,8\}$ can only appear as a whole block from the point-of-view of player $1$, even if the partition $\{5,6,7,8\}$ is not a single coalition.
\end{example}

Eq.\,\eqref{eq:OwenRecursivePart} applies to $m$-ary coalition structure, but the case for binary hierarchies is simpler.
By assuming $m=2$, the formula $\Omega_i(Q, T)$ of Eq.\,\eqref{eq:OwenRecursivePart} can be simplified, obtaining Eq.\,\eqref{eq:BPT} and completing our derivation. 

\subsection{Proof of Theorem \ref{th:cost}}
\label{app:proofCost}

Applying Eq.\,\eqref{eq:BPT} to a partition $T$ that admits a sub-coalition structure $T{\downarrow}=\{T_{1},\, T_{2}\}$ creates four branches (two for $i \in T_{1}$ and two for $i \in T_{2}$) and necessitates two $\nu$ evaluations. 
Since we are assuming the BHCS hierarchy to be a balanced tree with depth $d$, we can define the total number $a(d)$ of $\nu$ evaluations for the expansion of all nodes up to depth $d$.
Such quantity $a(d)$ follows a linear recurrence sequence represented by Eq.\,\eqref{eq:costRecur}
\begin{equation}\label{eq:costRecur}
    a(d) = \begin{cases}
        4 \cdot a(d-1) + 2 & \text{if }d > 0 \\
        0 & \text{if }d=0
    \end{cases}
\end{equation}
Recursion from Eq.\,\eqref{eq:costRecur} can be eliminated, since the equation is a well-known non-homogeneous linear recurrence with constant coefficients, having solution
\begin{equation}
    a(d) = \alpha\cdot a(d-1) + \beta = \frac{\beta(\alpha^{d-1} - 1)}{\alpha - 1}
\end{equation}
By using $\alpha=4$ and $\beta=2$, Eq.\,\eqref{eq:costRecur} simplifies to
\begin{equation}
a(d) = \frac{2}{3}(4^{d-1} - 1) \approx O(4^d)
\end{equation}
i.e., the time complexity of Eq.\,\eqref{eq:BPT} exhibits exponential growth.

\subsection{Pseudo-code of the Owen approximation algorithm}
\label{app:pseudoCodeOwen}

A limitation of equation Eq.\,\eqref{eq:BPT} is that the same coalitions are generated in the recursive expansion of $\Omega_i(\varnothing, \TT)$, for different players $i \in \setN$.
This issue may severely limit the performance, but it can be easily solved either by memoization, or by generating all the coalitions using a tree visit.
An efficient iterative implementation of the latter is sketched in Algorithm~\ref{algo:queueqPartExplain}, and it is conceptually equivalent to the Partition Explainer of SHAP\,\cite{shapPartitionExplainer}.
Therefore it does not constitute a novel paper contribution, but we report it for reader's convenience and self-containment. 

\LinesNumbered
\begin{algorithm}[t]
\caption{Iterative implementation of Eq.~\eqref{eq:BPT}.}
\label{algo:queueqPartExplain}
\Function{\FuncCall{OwenValues}{$\nu$, $\TT$, $b$}} 
{
    \KwSty{foreach} $i \in \setN$ ~~\KwSty{do}~~
        $\Omega[i] \gets 0$ \;

    $\mathit{queue}$.push$\bigl(
        \tuple{1, \varnothing, \TT, \nu(\varnothing), \nu(\setN)}
        \bigr)$ \;

    \While{$\mathit{queue}$ is not empty} {
      $w, Q, T, v_Q, v_{Q \cup T} \gets \mathit{queue}.\text{pop}()$ \;
      \If{$T$ is indivisible or $b \leq 1$} {
        \ForEach{$i \in T$}{
          $\Omega[i] \gets 
                \Omega[i] +
                    \frac{w}{|T|}\bigl( v_{Q \cup T} - v_Q \bigr)$ \;
        }
      }
      \Else {
          $T_1, T_2 \gets T{\downarrow}$ \;
          $v_{Q \cup T_1} \gets \nu(Q \cup T_1)$ \;
          $v_{Q \cup T_2} \gets \nu(Q \cup T_2)$ \;
          $b \gets b-2$ \label{alg:nueval} \;
          $\begin{array}{ll}
            \!\!\!\!\mathit{queue}.\mathrm{push}\big(\!\!\!\! &
             \tuple{\tfrac{w}{2}, Q, T_1, v_Q, v_{Q \cup T_1}},~~         \tuple{\tfrac{w}{2}, Q\cup T_2, T_1, v_{Q \cup T_2}, v_{Q \cup T}}, \\
             &\tuple{\tfrac{w}{2}, Q, T_2, v_Q, v_{Q \cup T_2}},~~ 
             \tuple{\tfrac{w}{2}, Q\cup T_1, T_2, v_{Q \cup T_1}, v_{Q \cup T}} ~~\big)
          \end{array}$ \;
      }
    }
    \Return{$\Omega$}
}
\end{algorithm}

Algorithm~\ref{algo:queueqPartExplain} operates at the partition level. 
It starts from the full coalition at the root $\TT$ of the BPT hierarchy (measuring the difference $\nu(\setN) - \nu(\varnothing)$). 
Partitions are inserted into a queue, assumed to be ordered by a priority $w$.
It then proceeds by splitting the next partition with the highest $w$, using Eq.\,\eqref{eq:BPT}.
Each split requires two model evaluations (line \ref{alg:nueval}), thus reducing the budget $b$ by $2$.
The splitting continues until the budget $b$ is consumed, or all partitions left are indivisible.

\LinesNumbered
\begin{algorithm}[ht]
\caption{Pseudo-code of the BPT algorithm.}
\label{algo:BPT}
\begin{small}
\Function{\FuncCall{init\_bpt}{$\mathcal{X}$:image}} 
{
    \ForEach{pixel $px$ of image $x$}{
      $i \gets \FuncCall{make\_partition}{}$ \;
      $\mathit{minR}[i] \gets \mathit{maxR}[i] \gets \mathit{R}[px]$ \;
      $\mathit{minG}[i] \gets \mathit{maxG}[i] \gets \mathit{G}[px]$ \;
      $\mathit{minB}[i] \gets \mathit{maxB}[i] \gets \mathit{B}[px]$ \;
      $\mathit{area}[i] \gets 1$;~
      $\mathit{perimeter}[i] \gets 4$;~
      $\mathit{root}[i] \gets i $ \;
    }
    \ForEach{pair of partitions $i,j$ that have adjacent pixels in $x$}{
      \FuncCall{heap\_push}{$heap$,
        \FuncCall{make\_adjacency}{$i$, $j$, weight=\FuncCall{get\_dist}{$i$, $j$}} }
    }
}
\algrule
\setcounter{AlgoLine}{0}
\Function{\FuncCall{get\_dist}{$i$, $j$}} 
{
  $\mathit{rangeR} \gets \max(\mathit{maxR}[i] - \mathit{maxR}[j]) - \min(\mathit{minR}[i] - \mathit{minR}[j])$ \;
$\mathit{rangeG} \gets \max(\mathit{maxG}[i] - \mathit{maxG}[j]) - \min(\mathit{minG}[i] - \mathit{minG}[j])$ \;
$\mathit{rangeB} \gets \max(\mathit{maxB}[i] - \mathit{maxB}[j]) - \min(\mathit{minB}[i] - \mathit{minB}[j])$ \;
  $\mathit{area} \gets \mathit{area}[i] + \mathit{area}[j]$ \;

  $\mathit{perimeter} \gets \mathit{perimeter}[i] + \mathit{perimeter}[j] -  2*\mathit{adjacent\_perimeter[i,j]}$ \;

  $\mathit{color\_score} \gets (\mathit{rangeR}^2 + \mathit{rangeG}^2 + \mathit{rangeB}^2)$ \;
  \Return{$\mathit{color\_score} * \mathit{area} * \sqrt{\mathit{perimeter}}$}
}
\algrule
\setcounter{AlgoLine}{0}
\Function{\FuncCall{build\_bpt}{}} 
{
  \While{$heap$ is not empty} {
    $adj \gets \FuncCall{heap\_pop}{heap}$ \;
    $i, j \gets \text{partitions in } adj$ \; 
    $k \gets \FuncCall{make\_partition}{}$ \;
    $\mathit{minR}[k] \gets \min(\mathit{minR}[i], \mathit{minR}[j])$ \; $\mathit{maxR}[k] \gets \max(\mathit{maxR}[i], \mathit{maxR}[j])$ \;
    $\mathit{minG}[k] \gets \min(\mathit{minG}[i], \mathit{minG}[j])$ \; $\mathit{maxG}[k] \gets \max(\mathit{maxG}[i], \mathit{maxG}[j])$ \;
    $\mathit{minB}[k] \gets \min(\mathit{minB}[i], \mathit{minB}[j])$ \; $\mathit{maxB}[k] \gets \max(\mathit{maxB}[i], \mathit{maxB}[j])$ \;
    $\mathit{area}[k] \gets \mathit{area}[i] + \mathit{area}[j]$ \;
    $\mathit{perimeter}[k] \gets \mathit{perimeter}[i] + \mathit{perimeter}[j]$ \;
    $\mathit{root}[k] \gets k$ ;~~
    $\mathit{root}[i] \gets \mathit{root}[j] \gets k$ \;
    $\mathit{left\_branch}[k] \gets i$ ;~~
    $\mathit{right\_branch}[k] \gets j$ \;
    merge linked lists of adjacencies of $i$ and $j$ into a single linked list for partition $k$, updating the heap weights using \FuncSty{get\_dist} since partitions $i$ and $j$ are now merged together. \label{line:mergeAdj}

  }
}
\end{small}
\end{algorithm}


\subsection{Pseudo-code of the BPT algorithm}
\label{app:pseudoCodeBPT}

Detailed pseudo-code for the BPT algorithm can be found in \cite{salembier2000BPT,randrianasoa2018multiFeatureBPT,randrianasoa2021agatBPT}, but a pseudo-code is provided in Algorithm~\ref{algo:BPT}.
It uses three functions:
\begin{itemize}
    \item \FuncSty{init\_bpt}: initializes the unitary partitions $i$ of the BPT hierarchy from the individual pixels $px$ of the input image $x$, and creates the heap of all the pairs of adjacent pixels.

    \item \FuncSty{get\_dist}: computes the distance between two (adjacent) partitions $i$ and $j$ using Eq.\,\eqref{eq:bpt:distance}.

    \item \FuncSty{build\_bpt}: iteratively merges adjacent partitions in \emph{distance}-order, each time creating a new merged partition $k$, and updates the weights in the heap accordingly. The function proceeds as long as there are adjacent partitions, i.e. it stops when all pixels are merged into a single root partition.
\end{itemize}

Once Algorithm~\ref{algo:BPT} has generated a \emph{merging sequence}, it can be efficiently stored into 6 arrays:
\begin{itemize}
  \item $\mathit{leaf\_idx}[i]$: the image pixel of unitary coalition $i$, with $i \in [1, n]$;
  \item $\mathit{left\_branch}[k]$ and $\mathit{right\_branch}[k]$: the two partition indexes resulting from the split $T_k{\downarrow}$ of each non-unitary coalition $k$, with $k \in [n+1, 2n-1]$;
  \item $\mathit{start}[k]$ and $\mathit{end}[k]$: index interval of pixels for the non-unitary partition $k$;
  \item $\mathit{pixels}$: the sorted array of pixel indexes, indexed by $\mathit{start}$ and $\mathit{end}$.
\end{itemize}
Therefore, the memory requirement for the BPT hierarchy is $\Theta(6n)$ integers.

The core data structure is a graph of the partitions (nodes), paired with the list of adjacencies (edges). The adjacency list needs to be sorted efficiently in order to extract the edge $\mathit{adj}=(i,j)$ having the smallest $dist(i,j)$, as defined by Eq.\,\eqref{eq:bpt:distance} and computed by function \FuncSty{get\_dist}.
To do so, a heap data structure is a reasonable choice.
Merging coalitions therefore requires to both modify the nodes and update the edges.
This process, described at line \ref{line:mergeAdj} of \FuncSty{build\_bpt} and depicted in Figure~\ref{fig:BPT}/B, shows that each merge operation requires to traverse the adjacency list of the merged partitions.
Further details can be found in \cite{randrianasoa2018multiFeatureBPT}.


\subsection{Python implementation}

\emph{ShapBPT} is implemented in Python.
A snippet of the python code using the \emph{ShapBPT} package to obtain a Shapley explanation for a given image using the masking function $\nu$ is provided in Algorithm~\ref{algo:shapBPTpythonCode}. 
While not detailed in the paper, the implementation supports multi-class explanations, similarly to \cite{shapPartitionExplainer}.

\LinesNumbered
\begin{algorithm}[ht]
\caption{Example Python code.}
\label{algo:shapBPTpythonCode}
\KwSty{from} shap\_bpt \KwSty{import} \textit{Explainer} \;
explainer = \textit{Explainer}($\nu$, image\_to\_explain, num\_explained\_classes) \;
shap\_values = explainer.\textit{explain\_instance}(max\_evals=$b$)
\end{algorithm}


\subsection{Sensitivity of the distance function}
\label{app:dist_sensitivity}

We run a small sensitivity analysis of the distance function
$\mathit{dist}(T_i,T_j)$ over 100 randomly sampled images from
ImageNet‐S$_{50}$.  Unless otherwise stated, all experiments use the
same \textit{ResNet-50} classifier, a fixed evaluation budget $b$ of 100
model calls per image and the ShapBPT hyper-parameters reported
in the main text.

We consider three variations of the distance function.
\begin{itemize}
    \item \emph{Default} (\emph{Eq.\,\eqref{eq:bpt:distance}}) - color\,$\times$\,area\,$\times\sqrt{\text{perimeter}}$.
    \item \emph{No-perimeter} - color\,$\times$area (drops the perimeter term).
    \item \emph{No-color} - area\,$\times\,\sqrt{\text{perimeter}}$ (drops the color term).
\end{itemize}
Area cannot be dropped, as it generates imbalanced trees.
We report the \textit{relative} change (\(\Delta\%\)) in $\AUIoU$ and $\maxIoU$ against the default distance.

\begin{center}
  \begin{tabular}{lrr}
    \hline
    \textbf{Distance variant} & $\Delta\AUIoU$ $\uparrow$ & $\Delta\maxIoU$ $\downarrow$ \\
    \hline
    Default (Eq.\,\eqref{eq:bpt:distance})
        & 0.0\% 		& 0.0\% \\
    No-perimeter term
        & $-2.65$\%		& $-3.78$\% \\
    No-color term
        & $-12.61$\%	& $-24.40$\% \\
    \hline
  \end{tabular}    
\end{center}

Dropping the perimeter term produces a small loss in $\AUIoU$ of $-2.65$\% and a small loss in $\maxIoU$ of $-3.78$\%, showing that the presence of the perimeter term provides a benefit.
Dropping the color term results in significant losses ($-12.61$\% in $\AUIoU$ and $-24.40$\% in $\maxIoU$), which shows that color term is very relevant.
Therefore, the default color-area-perimeter distance of Eq.\,\eqref{eq:bpt:distance} is a well-behaved compromise: inexpensive to compute and close to the Pareto front.

We plan to study more complex alternatives in a future work.

\clearpage
\subsection{Evaluation details}

\subsubsection{Experiment E1}
\label{app:Experiment-E1}

This experiment employs the \emph{1K-V2} pretrained model\,\cite{vryniotis2021train}, utilizing the ResNet50 architecture\,\cite{resnet50} available in the PyTorch library, which reports an accuracy of $80.858\%$. Masking is applied by substituting affected pixels with a uniform gray color. The analysis is conducted on the ImageNet-S$_{50}$ dataset\,\cite{gao2022lussImageNetS}, which provides precise ground-truth masks for a selected subset of images. To maintain consistency, only images for which the ground-truth mask corresponds to the top predicted class are considered, resulting in a total of 574 images.

\begin{figure}[H]
    \centering
    \includegraphics[width=1.0\linewidth]{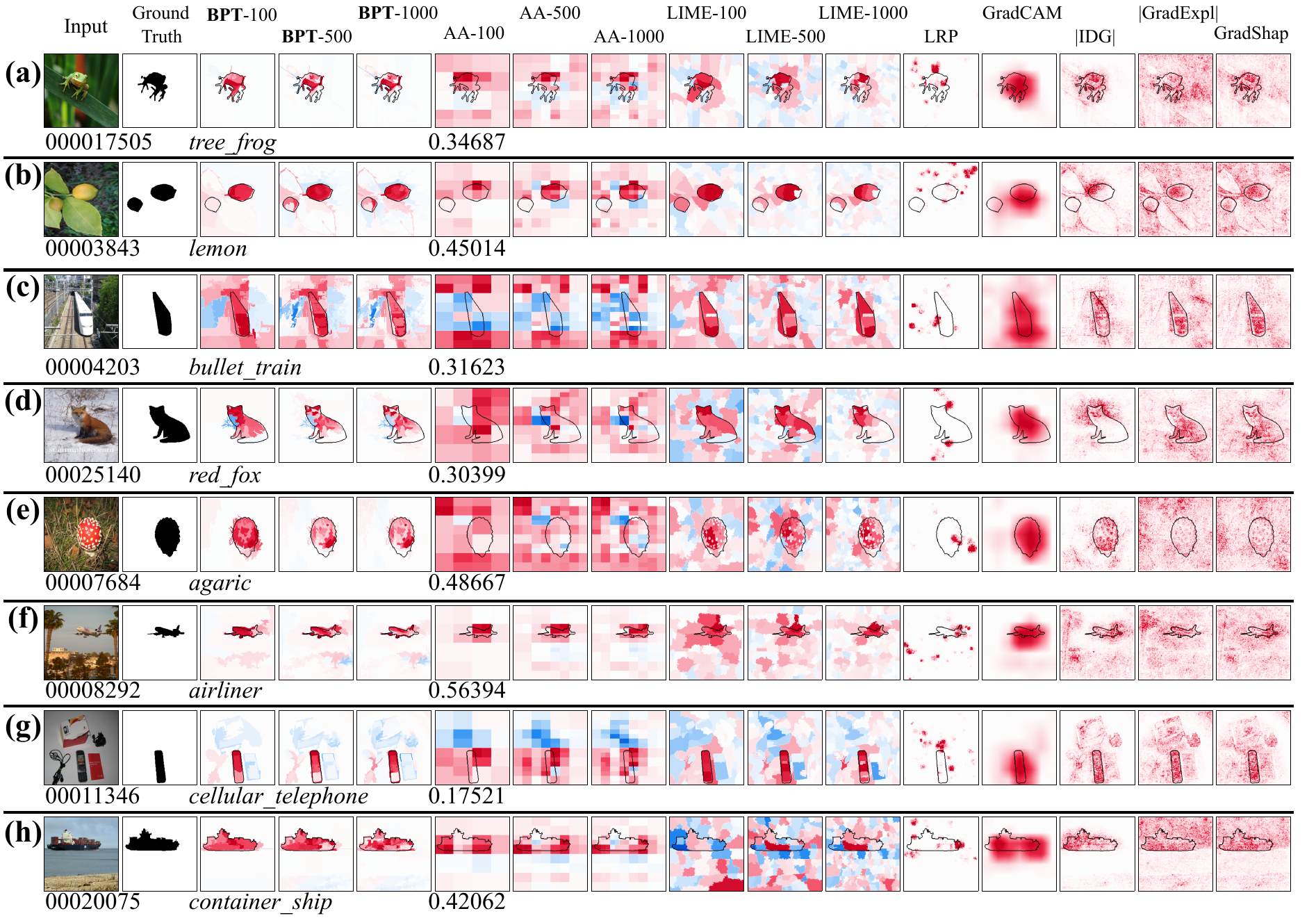}
    \caption{Additional saliency maps generated for the \textbf{E1} experiment.}
    \label{fig:saliencyMapsExt-E1}
\end{figure}

Figure\,\ref{fig:saliencyMapsExt-E1} shows additional saliency maps for the \textbf{E1} experiment, generated by explaining the classification of the ResNet50 model on the samples from the ImageNet-S$_{50}$ dataset.

Figure~\ref{fig:resultsE1} reports the results for \textbf{E1}, with one table for each of the four scores, plus one for the evaluation time (logscale).
All reported times were computed with an Intel Core i9 CPU, an Nvidia 4070 GPU, and 16GB of RAM.
Scores are drawn as boxplots (treating values outside 10 times the interquantile range as outliers, drawn as fuchsia dots), with a method symbol on the right (see the legend for the mapping).

\begin{figure}[h]
    \centering
    \includegraphics[width=1.0\linewidth]{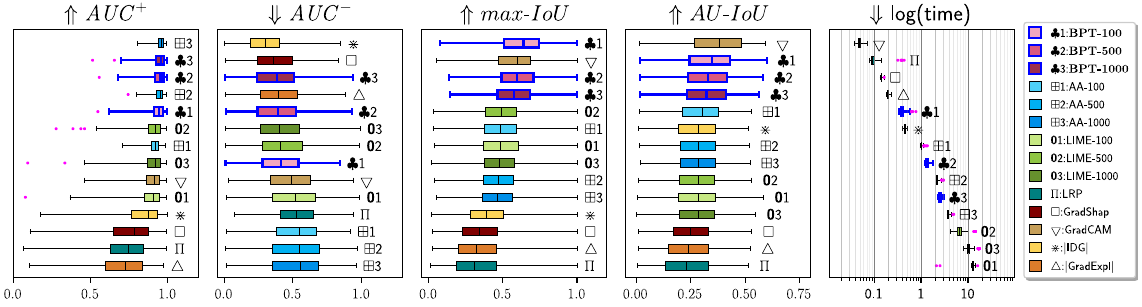}
    \caption{Results for four metrics across 574 images from the ImageNet-S$_{50}$ dataset, with methods ranked by performance (highest at the top) for experiment E1. Arrows denote whether higher or lower scores are better.}
    \label{fig:resultsE1}
\end{figure}

In \textbf{E1}, BPT is positioned close or at the top of every score.
In this case, AA has a slightly better $\mathit{AUC}^{+}$ score, but a worse $\mathit{AUC}^{-}$ score than BPT.
The BPT method seems to be particularly effective at the IoU scores \emph{max-IoU} and \emph{AU-IoU}, which can be explained by its capacity of recognizing the borders of the objects, by following a data-aware hierarchy.
Only GradCAM reaches similar IoU scores, but in practice the localization of GradCAM is more blurred and fuzzy (this limitation is apparently not well captured by the two IoU scores).

\subsubsection{Experiment E2}
\label{app:additionalExperiment-E2}
One important limitation of experiments relying on some unknown black-box model is that the ground truth may not be faithful, as the model may classify an object based on partial details or using weak correlations.
To overcome this limitation, experiment \textbf{E2} replicates \textbf{E1} adopting an ideal model which perfectly follows the ground truth.

The ideal model
\begin{equation}
    \nu_\text{lin}(S) = \tfrac{|S \cap G|}{|G|}
\end{equation}
is a linear function that outputs the proportion of pixels of $S$ that belong to the ground truth $G$.
Since $\nu_\text{lin}$ is not a neural network, CAM methods cannot be used and are excluded.
By using a linear model, the experimental environment has minimal noise, is therefore simpler to interpret, and provides a better baseline for assessment, even if it is less realistic than a deep learning model.

\begin{figure}[H]
    \centering
    \includegraphics[width=1.0\linewidth]{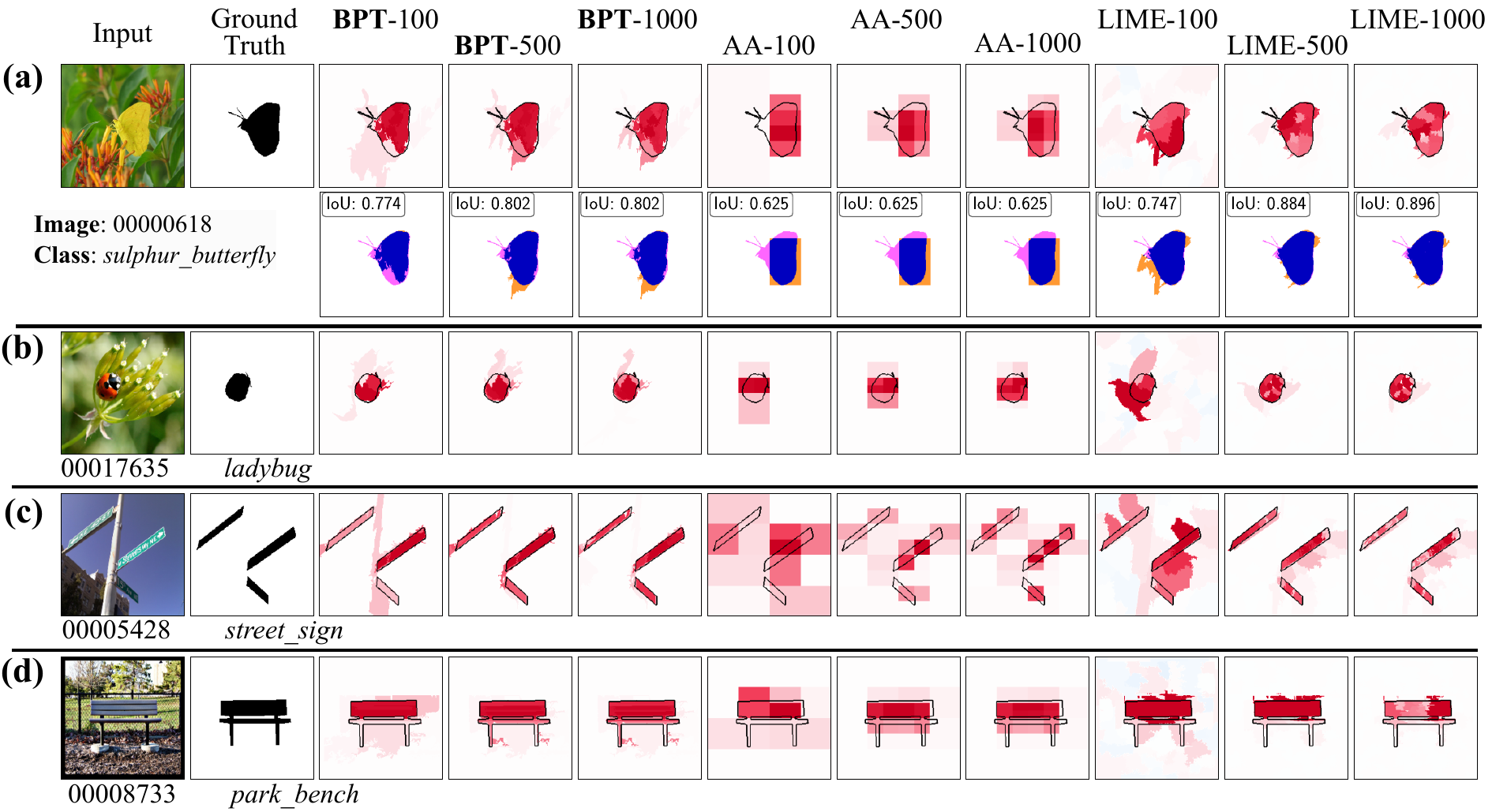}
    \caption{Saliency maps obtained from the ideal linear model $\nu_\text{lin}$, experiment \textbf{E2}.}
    \label{fig:Exp_E2}
\end{figure}

Figure \ref{fig:resultsE2} shows the results of experiment \textbf{E2}, while a subset of the generated saliency maps are depicted in Figure\,\ref{fig:Exp_E2}.
The results shows the effectiveness of the BPT explanation strategy: all BPT-$b$ achieve better scores that their AA-$b$ counterpart, for the same budget $b$.

\begin{figure}[H]
    \centering
    \includegraphics[width=1.0\linewidth]{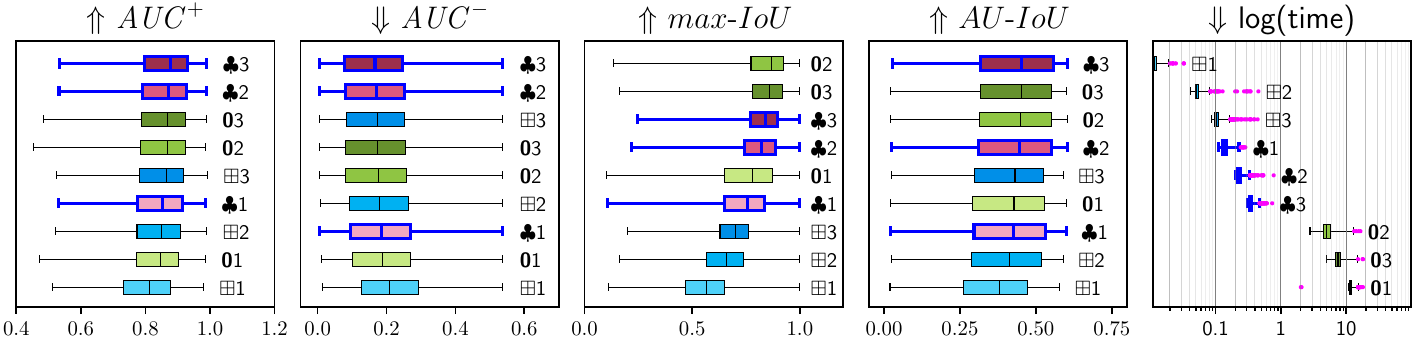}
    \caption{Results for the four metrics across 574 images from the ImageNet-S$_{50}$ dataset, with methods ranked by performance (highest at the top) for the experiments \textbf{E2}.} 
    \label{fig:resultsE2}
\end{figure}


\subsubsection{Experiment E3}
\label{app:additionalExperiment-E3}

Experiment \textbf{E3} replicates the setup of \textbf{E1} and \textbf{E2}, but employs a Vision Transformer model, specifically \emph{SwinViT}\,\cite{liu2021swin}. 
Vision Transformer models are known for their robustness against partial occlusion of recognized objects, making it more challenging for model-agnostic methods to analyze their behavior by selectively masking parts of the image. A summary of the results is presented in Figure \ref{fig:results_table_swin_trans_vit_real_gray}, while a selection of saliency maps from the same set of examples is depicted in Figure \ref{fig:heatmaps_E3_Swin_ViT}.

Due to the limitations of the LRP method's implementation, which does not support this transformer-based architecture, we excluded it from the results. Analyzing the explanations produced by different methods, it is evident that all approaches, except for BPT, generate significantly more ambiguous saliency maps, attributing considerable importance to background features while failing to focus adequately on the classified objects. In contrast, the maps produced by BPT appear clearer and more focused. Notably, BPT consistently achieves superior performance across all evaluation metrics.

\begin{figure}[H]
    \centering
    \includegraphics[width=\linewidth]{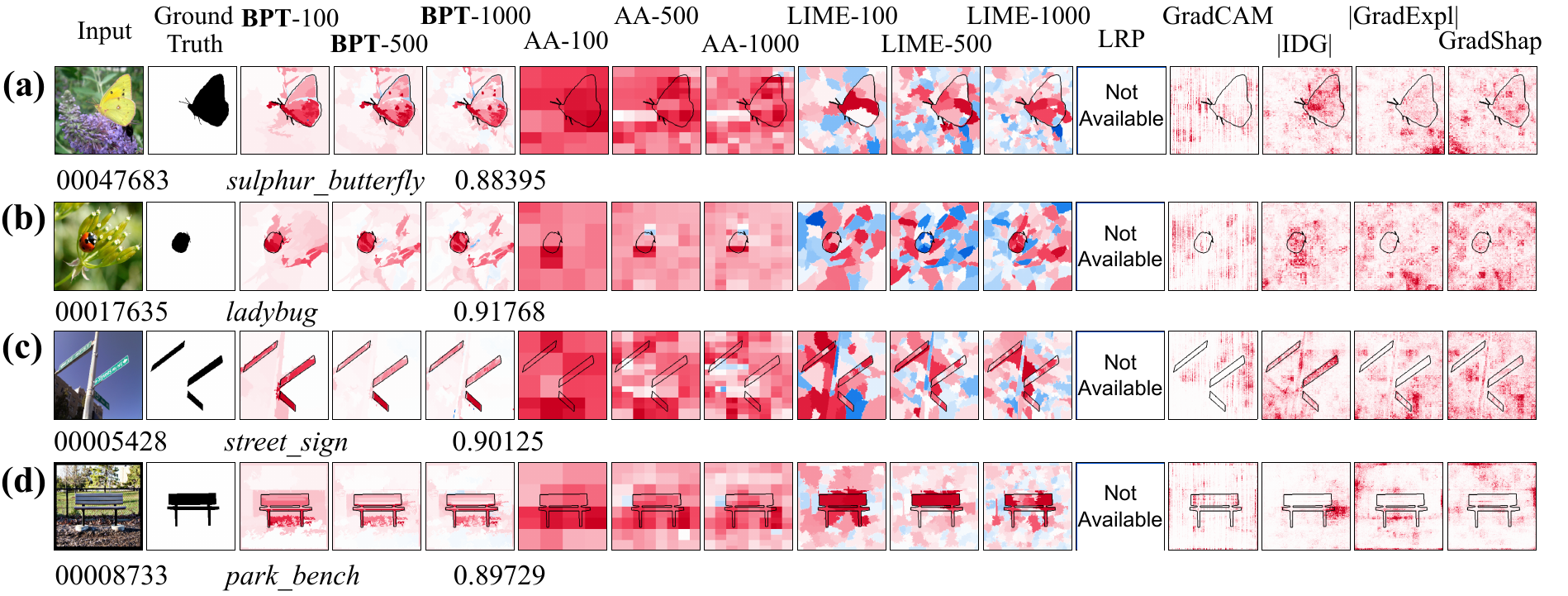}
    \caption{Saliency maps from selected instances in the \textbf{E3} experiment (with SwinViT).}
    \label{fig:heatmaps_E3_Swin_ViT}
\end{figure}

This experiment provides valuable insights, as Vision Transformer models exhibit increased robustness to input masking, making them particularly challenging to interpret using model-agnostic methods. Unlike convolutional models, these transformers-based model require clever feature replacement techniques for behavior probing, and ShapBPT seems to be significantly better than the other methods.

\begin{figure}[H]
    \centering
    \includegraphics[width=\linewidth]{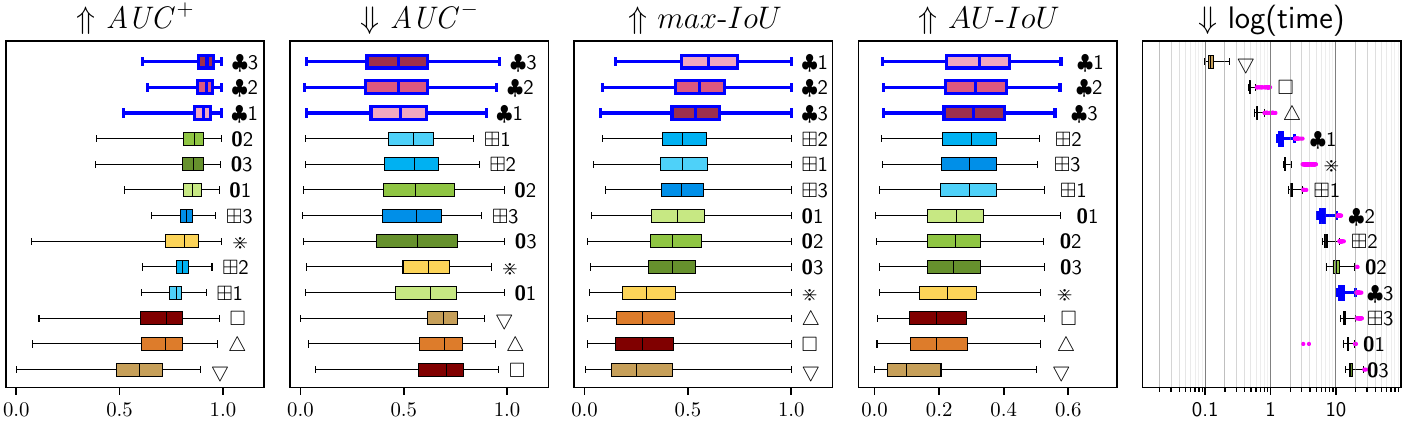}
    \caption{Results for the four metrics across 621 images from the ImageNet-S$_{50}$ dataset, with methods ranked by performance (highest at the top) for the experiments \textbf{E3}.}
    \label{fig:results_table_swin_trans_vit_real_gray}
\end{figure}


\subsubsection{Experiment E4}
\label{app:additionalExperiment-E4}

This experiment focuses on the MS-COCO dataset\,\cite{MSCOCO}, which includes 80 object classes with diverse image sizes and aspect ratios, and a wider range of details than ImageNet.
The task is object detection, evaluated on 5,000 images (from the validation set) with bounding box and segmentation map annotations.

\begin{figure}[H]
    \centering
    \includegraphics[width=0.90\linewidth]{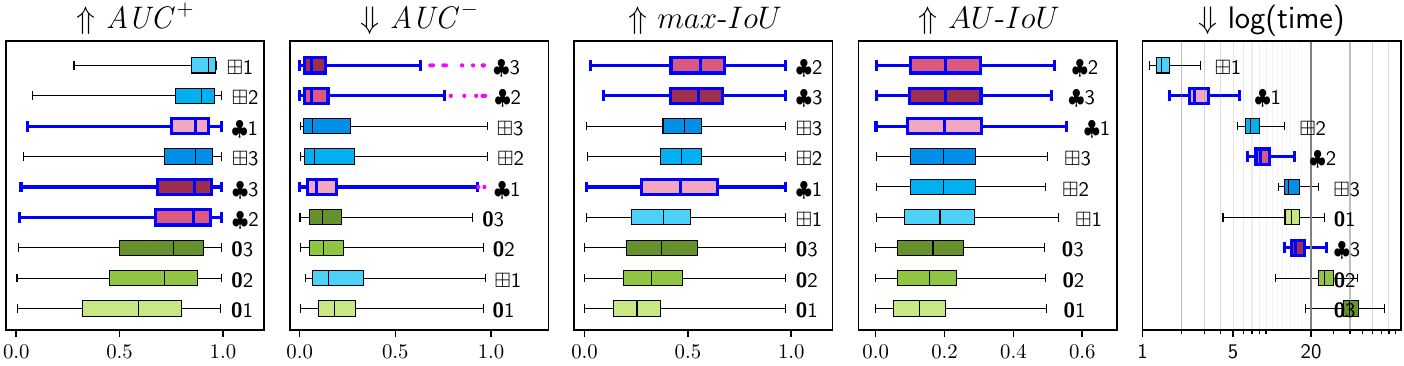}
    \caption{Results for the 274 test images from MS-COCO, with methods ranked by performance, for the experiments \textbf{E4}.}
    \label{fig:E4-results_MSCOCO}
\end{figure}
A pre-trained Yolo11s model\,\cite{yolo11ultralytics} from the Ultralytics library (319 layers, 9.4M parameters, 21.7 GFLOPs) is used as a black-box model for its accuracy and fast inference. The XCV task involves highlighting detected objects and comparing them to segmentation maps, with evaluation based on performance curve metrics and ground-truth comparisons, as explained in Experimental Assessment Section.

Also in this case, ShapBPT demonstrates strong capability in highlighting the boundaries of detected objects, outperforming AA in most metrics.

\begin{figure}[H]
    \centering
    \includegraphics[width=1.0\linewidth]{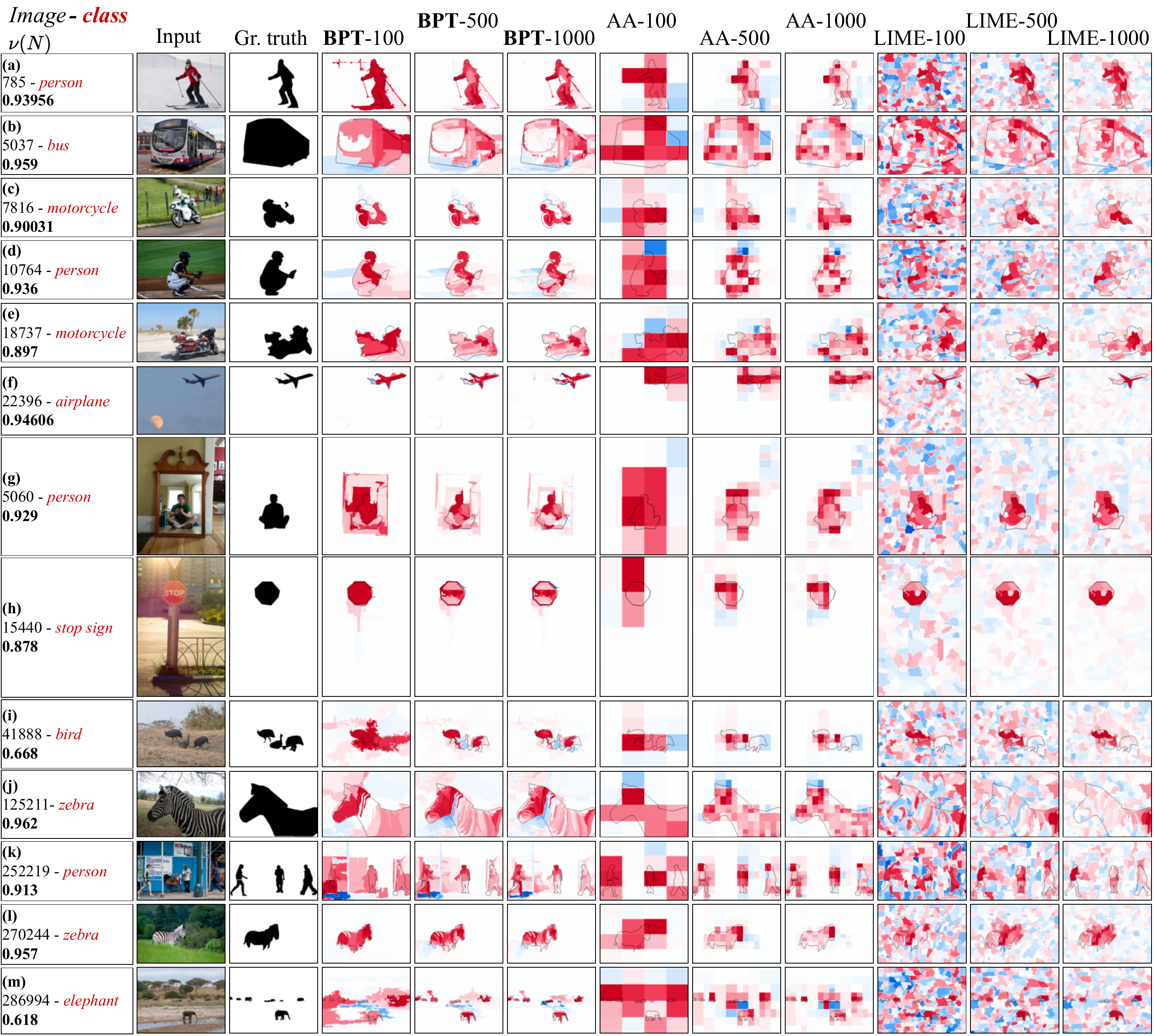}
    \caption{Examples of the \textbf{E4} experiment,  using the MS-COCO dataset.}
    \label{fig:E4-MS-COCO-examples}
\end{figure}

\subsubsection{Experiments E5}
\label{app:additionalExperiment-E5}

This experiment considers a multiclass regression model rather than a classification model. The objective is to determine the presence (positive prediction) or absence (negative prediction) of specific facial features, such as brown hair or eyeglasses. The explainable AI task involves identifying the regions that contribute to these predictions.
A score $\nu(\setN) > \nu(\varnothing)$ indicates the presence of the feature, while a score $\nu(\setN) < \nu(\varnothing)$ signifies its absence.

The dataset used for this study is CelebA-HQ \cite{karras2018progressive}, which includes 40 facial attributes. For this analysis, we focus on two attributes—brown hair and eyeglasses—for which ground-truth segmentation masks are available. A total of 106 images were evaluated. The model employed is a pre-trained sequential convolutional neural network (CNN) provided by \cite{Kartik}.

An example of the XCV task is illustrated in Figure \ref{fig:E5-CelebA-examples}, where multiple instances are analyzed. The first three are:
\begin{enumerate}
    \item A subject with brown hair, correctly identified as having brown hair (positive score).
    \item A subject with black hair, correctly identified as not having brown hair (negative score).
    \item A subject wearing eyeglasses, correctly identified as having them (positive score).
\end{enumerate}
For positive cases (a, c, d, and e), the Shapley values are positive in the regions contributing to the positive prediction. Conversely, for negative cases (b and f), the Shapley values are negative in the areas responsible for the negative prediction. 
Since CAM methods do not inherently satisfy the efficiency axiom of Shapley values, their outputs are considered in absolute terms.

\begin{figure}[H]
    \centering
    \includegraphics[width=\linewidth]{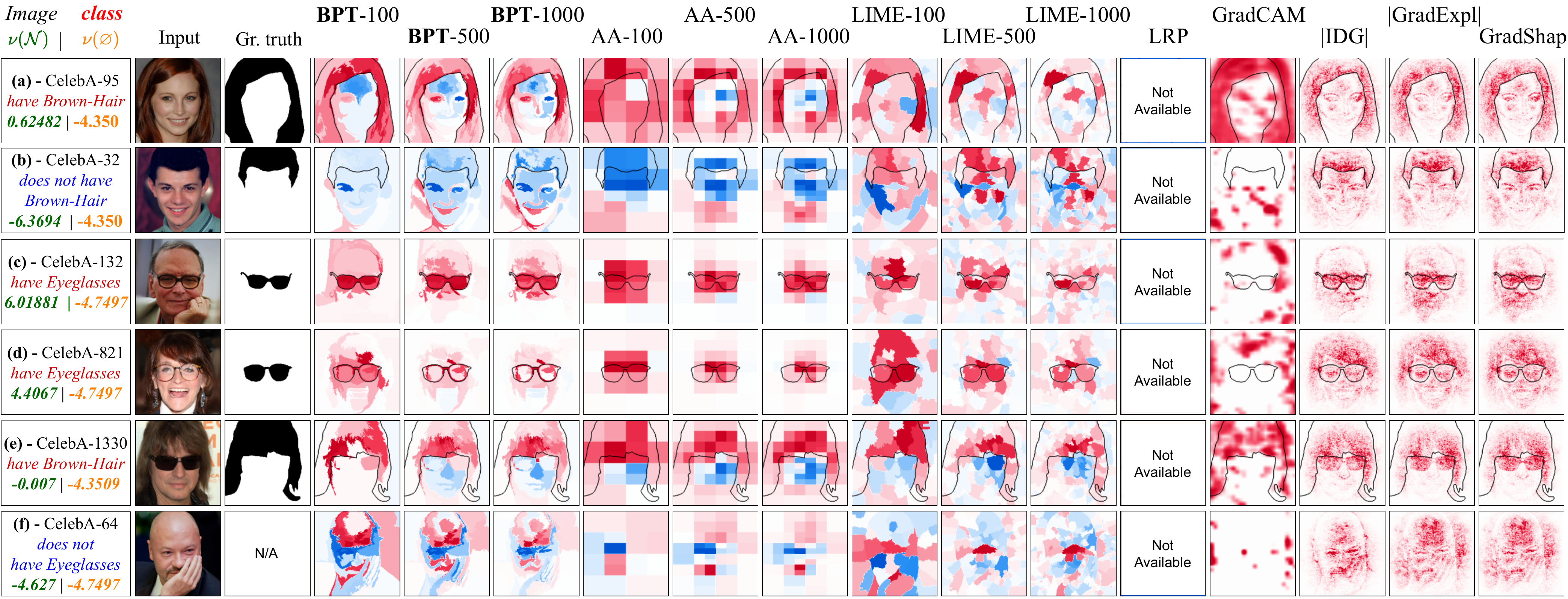}
    \caption{Examples of the \textbf{E5} experiment, explaining facial attributes using the CelebA dataset.}
    \label{fig:E5-CelebA-examples}
\end{figure}

\begin{figure}[H]
    \centering
    \includegraphics[width=1.0\linewidth]{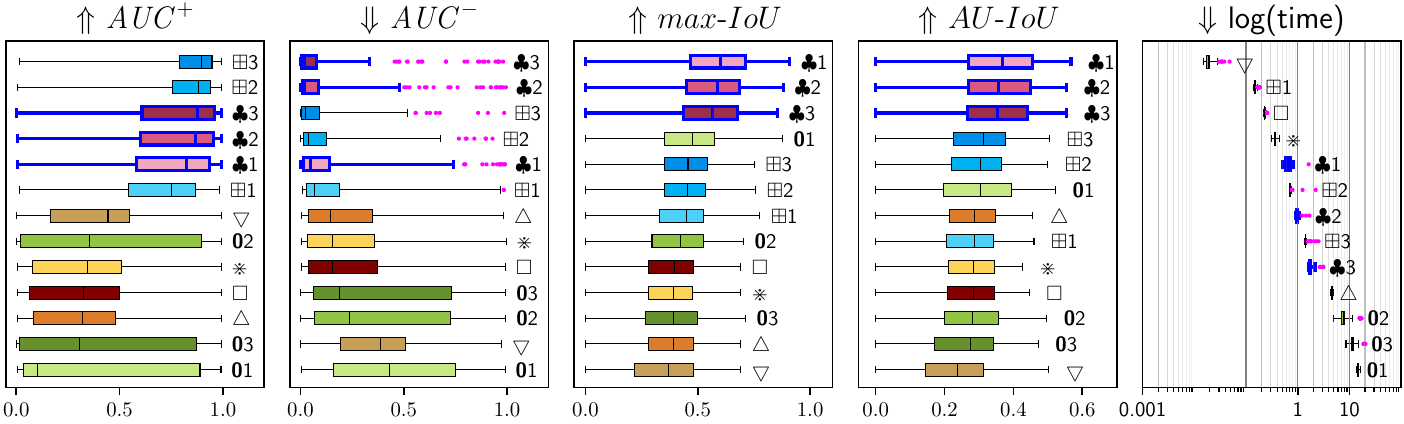}
    \caption{Results for the four metrics for the \textbf{E5} experiment on the CelebA dataset on 400 images.}
    \label{fig:E5-results_CelebAIoU}
\end{figure}

Results of the evaluation are reported in the tables in Figure\,\ref{fig:E5-results_CelebAIoU}. This experiment shows again the capacity of BPT-based methods to adaptively follow the borders of the activating regions, achieving high performances particularly on IoU scores.
Note that also in this case, as previously discussed for \textbf{E1}, the ground truth can only be considered as a weak approximation of the model's learnt representation, as the model is likely to use multiple features of the subject face to determine the presence or the absence of a specific attribute, not just the shape of the hair or the eyeglasses. Nonetheless, the localization of that area remains more precise when data-awareness is used.


\subsubsection{Experiments E6}
\label{app:additionalExperiment-E6}

Experiment \textbf{E6} considers the problem of explaining anomalies detected by an Anomaly Detection (AD) system on image data.
This experiment is based on the work of \cite{ravi2021general} where anomalies in images are detected using a Variational AutoEncoder-Generative Adversarial Network (VAE-GAN) model by means of anomaly localization. 
We use the MVTec benchmark dataset \cite{bergmann2019mvtec} which has 5000 high quality images with defective and non-defective samples from 15 different categories of objects. We selected the \textit{hazelnut} object category from the dataset.

\begin{figure}[H]
    \centering
    \includegraphics[width=0.70\linewidth]{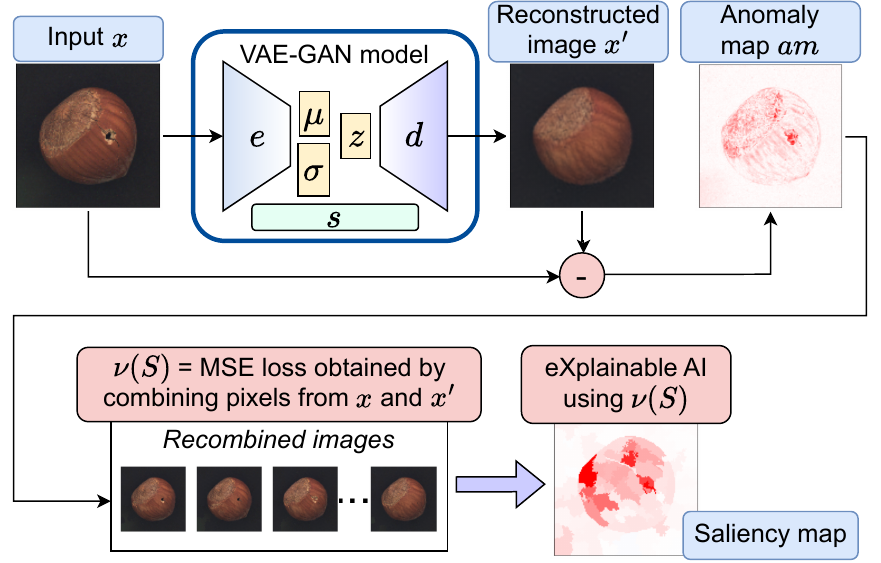}
    \caption{Workflow of the explainable AI applied to the Anomaly Detection system of \textbf{E6}.}
    \label{fig:AD}
\end{figure}

The pipeline of this system is depicted in Figure \ref{fig:AD}. An input image $x$ is reconstructed into $x'$ using a one-class VAE-GAN classifier. 

\begin{figure}[H]
    \centering
    
    \includegraphics[width=1.0\linewidth]{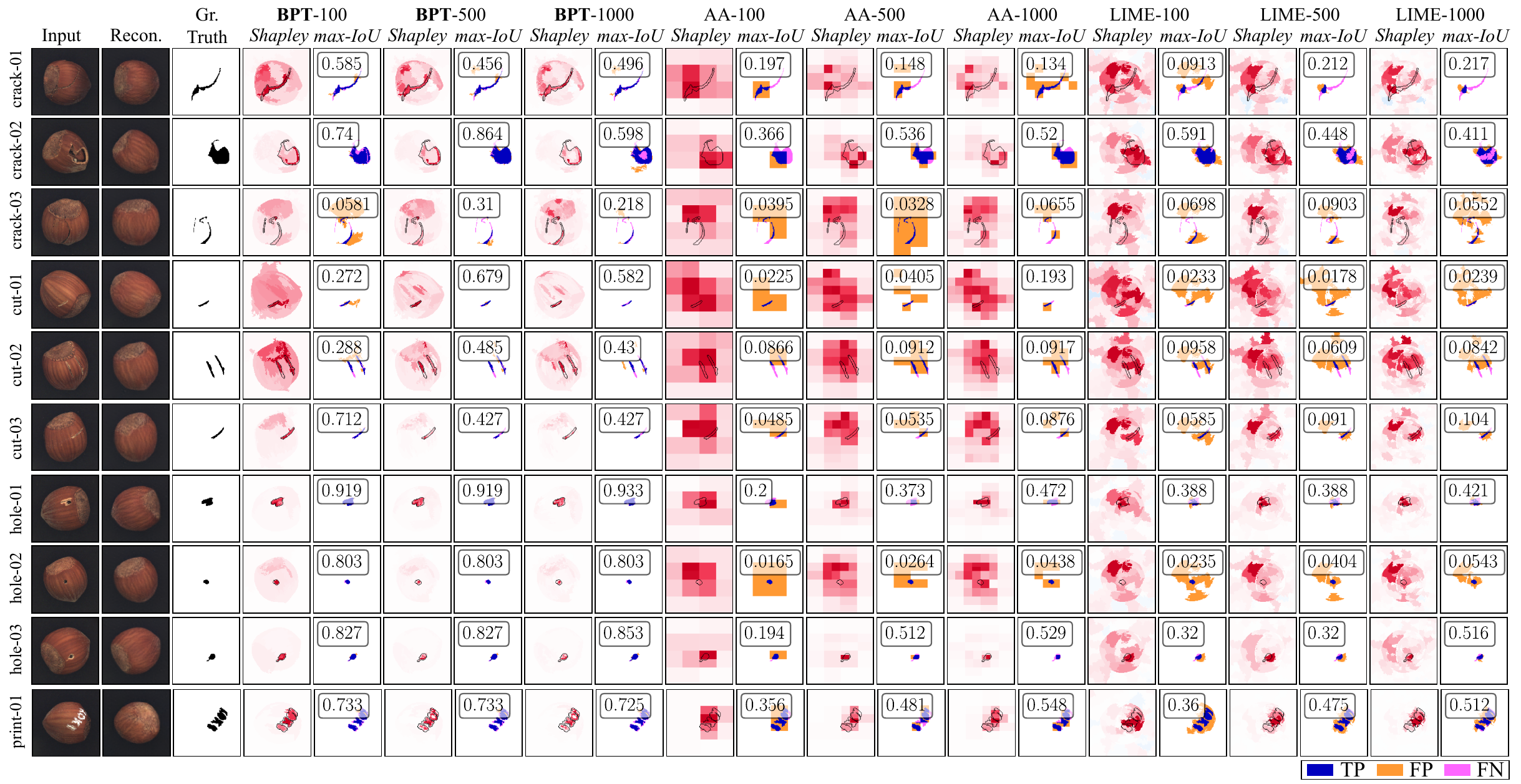}
    
    \caption{Selected examples in the Anomaly Detection system for experiment \textbf{E6}.}
    \label{app:fig_heatmaps-E6}
\end{figure}

The anomaly map $am$ captures the reconstruction error, which sums up both the potential anomalies of $x$ as well as the noise. An XAI method can be employed to separate the noise from the detected anomalies, thus localizing if and where the anomalies are present.
In this context, the function $\nu(S)$ is a MSE loss on the anomaly map $am$ itself. Since $\nu(S)$ is a MSE loss and not a neural network, CAM methods cannot be used.
Therefore, we generate saliency maps using BPT, AA and LIME.
We use values $100$, $500$, and $1000$ for the budget value $b$. 
For LIME, we use $100$, $500$ and $1000$ a-priori segments, respectively.

\begin{figure}[H]
    \centering
    \includegraphics[width=\linewidth]{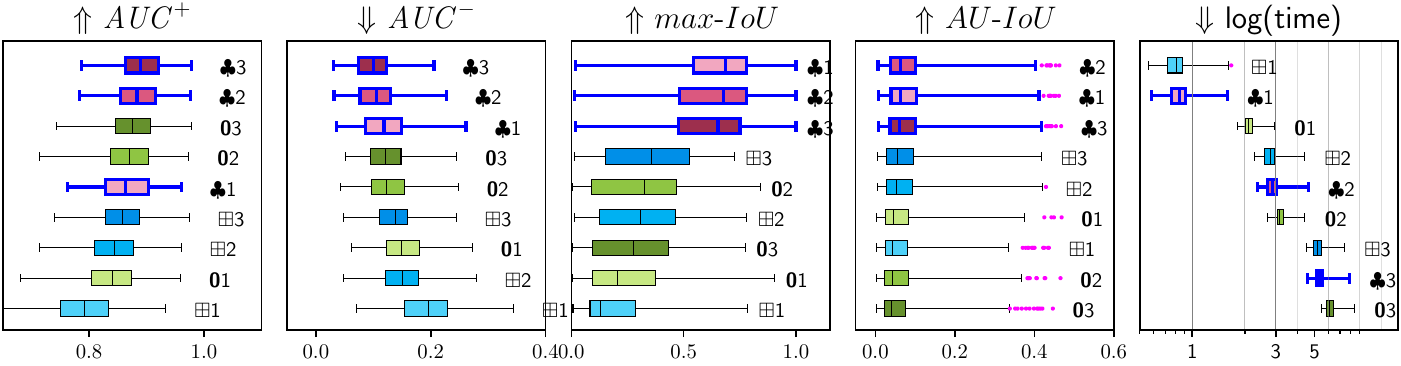}
    \caption{Results for the 4 metrics for the experiments \textbf{E6} on 280 images.}
    \label{fig:results-E6}
\end{figure}

As the MVTech dataset has proper ground truth masks for the expected anomalous regions, we can compute all the four scores defined in the Experimental Assessment Section. 

Figure \ref{app:fig_heatmaps-E6} shows the AD problem on three input images. For each input, a row shows: 
the input $x$, its reconstruction $x'$ through the VAE-GAN model, the anomaly map $am$, the explanation generated by BPT with $b{=}500$, by AA with $b{=}500$ and by LIME with $b{=}500$ and $100$ segments. The best intersection-over-union is also shown, highlighting the true positives (TP), the false positives (FP) and the false negatives (FN).
The ground truth $G$ is also shown, for reference.

Results are reported in Figure\,\ref{fig:results-E6}.
Again, all three XCV methods are capable of identifying the anomalous regions on the various samples, but BPT significantly outperforms the others. 
This is particularly true for the task of identifying the exact region, which is highlighted by the very high \textit{max-IoU} scores.

\subsubsection{Experiments E7}
\label{app:additionalExperiment-E7}

\begin{figure}[H]
    \centering
    \includegraphics[width=1.0\linewidth]{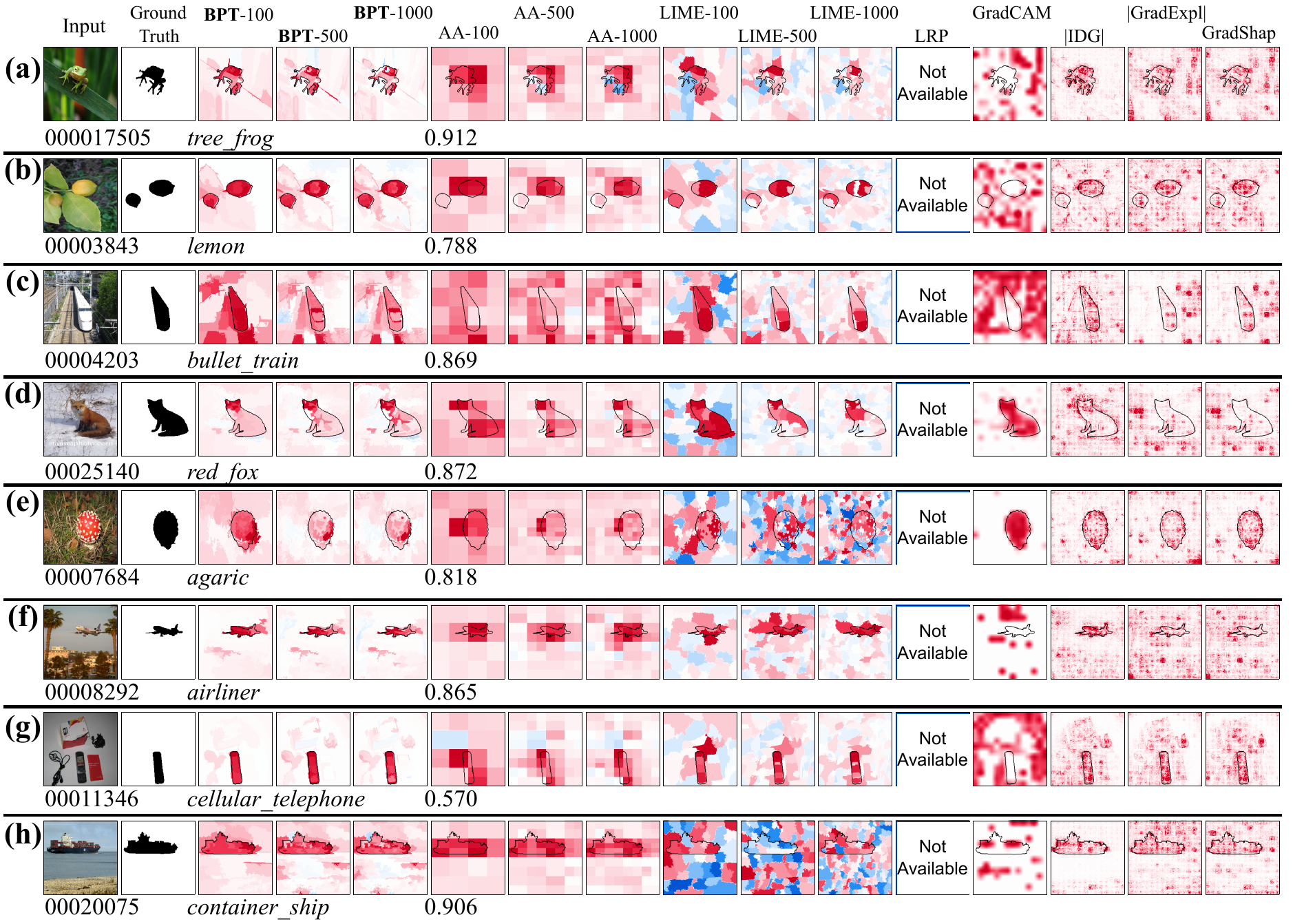}
    \caption{Saliency maps from selected instances in the \textbf{E7} experiment (with ViT-Base 16).}
    \label{fig:heatmaps_E7_ViT}
\end{figure}

Similar to E1 using the ViT-Base 16 model\,\cite{dosovitskiyimage}.
Saliency maps are illustrated in Figure\,\ref{fig:heatmaps_E7_ViT}.
Results are reported in Figure\,\ref{fig:results-E7}.
Similarly to E3 we observe a tendency of BPT to show more focused explanations than AA, and BPT achieves significantly better automated scores than AA.

\begin{figure}[H]
    \centering
    \includegraphics[width=\linewidth]{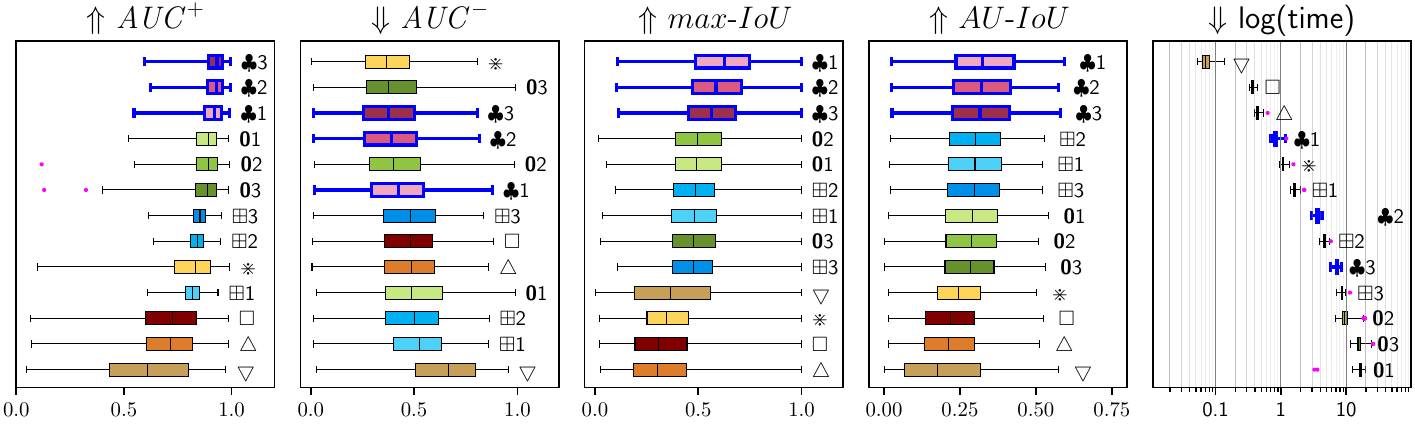}
    \caption{Results for the 4 metrics for the experiments \textbf{E7} on 593 images.}
    \label{fig:results-E7}
\end{figure}

\clearpage\clearpage



\subsubsection{E8 - User Preference Study}
\label{app:experiment-E8}

The study was conceived to quantify \emph{perceived intuitiveness and usefulness} of four explanation techniques:
\textbf{AA}-1000, \textbf{BPT}-1000, \textbf{LIME}-1000 and \textbf{GradCAM},
when applied to vision classification tasks of varying semantic granularity.
Because subjective preference is hard to measure on an absolute scale, we employed a \emph{within-subject, forced-ranking} design: each participant was shown four explanations side-by-side for a given image-classification task and asked to order them from most to least helpful. Saliency maps were permuted across tasks to mitigate position and label biases, while keeping the identity of each method hidden to every participant.
This design yields ordinal data that permit robust, non-parametric comparisons without assuming interval-scale judgments.

\begin{figure}[ht]
    \centering
    \includegraphics[width=0.5\linewidth]{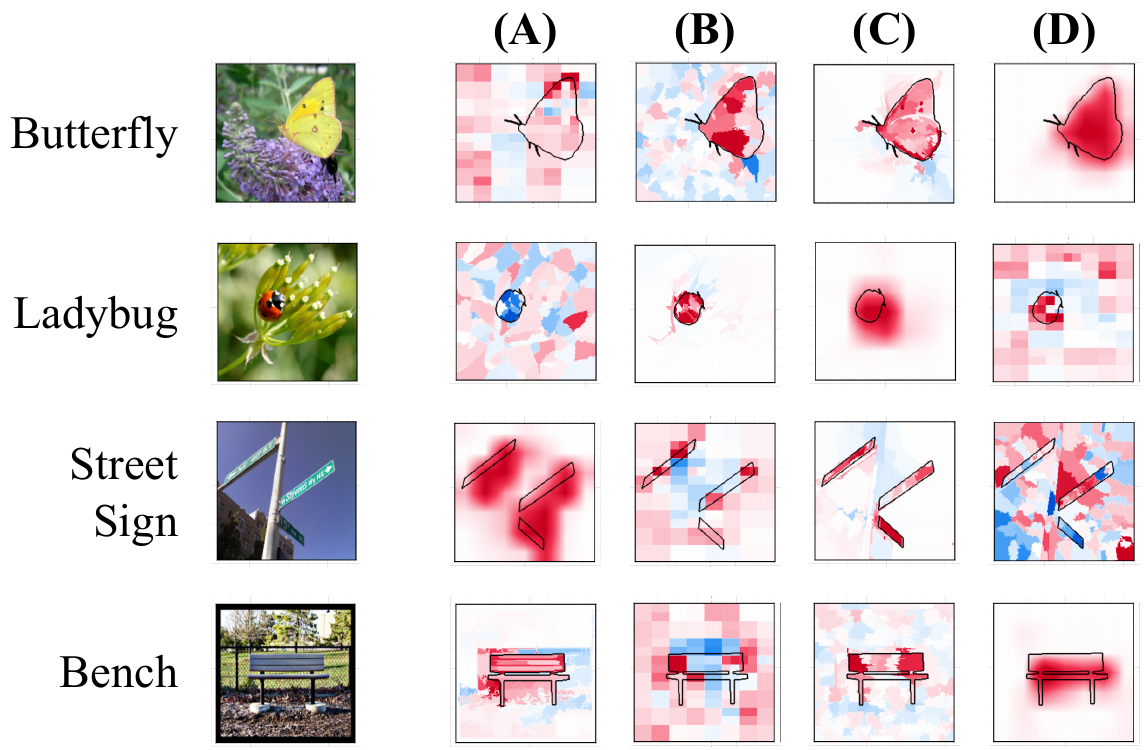}
    \caption{Sample structure of the questionnaire given to the 20 human subjects.}
    \label{fig:humanstudy}
\end{figure}

\paragraph{Experimental setting.}
Twenty volunteers (S1-S20) recruited from graduate students completed four ranking trials corresponding to the “\emph{Butterfly}”, “\emph{Ladybug}”, “\emph{StreetSign}”, and “\emph{Bench}” images. 
Recruited students have technical background in Machine Learning, but not in-depth knowledge of explainability and its associated methods.
Each trial presented the raw image and four colour-coded saliency maps produced by the hidden methods.  
Participants gave a total of $20 \times 4 = 80$ rank vectors and thus $80$ independent judgments for every method (see a sample in Figure\,\ref{fig:humanstudy}).  
After the ranking stage, subjects confirmed they understood the task.

\begin{figure}[ht]
    \centering
    \includegraphics[width=0.5\linewidth]{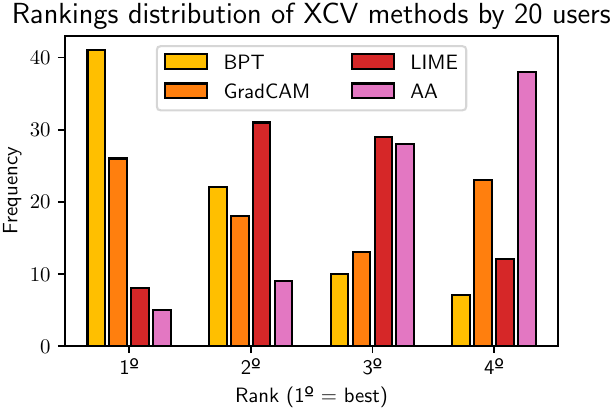}
    \caption{Distribution of the rankings across the four methods given by the 20 user subjects.}
    \label{fig:userrankings}
\end{figure}

\paragraph{Results and discussion.}
Figure\,\ref{fig:userrankings} summarises the outcome.  
BPT dominates the leftmost cluster, being selected first in $41/80 = 51\,\%$ of the trials and never relegated to fourth more than $7$ times; its mean rank is $\bar r_\text{BPT}=1.79$.  
GradCAM follows with a respectable first-place share ($33\,\%$) and $\bar r=2.41$; 
LIME is more evenly spread across the second and third positions ($\bar r=2.56$), while AA is strongly skewed to the bottom ($38/80$ fourth-place votes, $\bar r=3.24$).  

A Friedman test ($ \chi^2_{(k=3)} {=} 19.56,\; p{=}0.0002<0.05$, $H_0$ rejected) confirms that at least one explanation technique is ranked differently from the others with very high statistical confidence. 
To determine the significance, we follow with a Nemenyi pairwise comparison and identify BPT-vs-GradCAM and BPT-vs-LIME as marginally significant ($p{=}0.079$ and $p{=}0.092$, marginal at $p < \alpha {=} 0.10$),
and BPT-vs-AA as strongly significant ($p{=}0.000081$).
Taken together, the data argue that human users consistently find BPT explanations clearer and more intuitive, supporting its adoption as the default interpretability method in similar visual-classification pipelines.
Further discussions with the subject confirms the preference due to the clear and crisp presentation, with little noise and not blurred.
While results are statistically significant, we acknowledge that this user study is small and preliminary, and we will make a larger study as a future work.


\subsection{Limitations of SAM for Owen-style Shapley attribution.}

The \emph{Segment Anything Model} (SAM) provides, via its \emph{SamAutomaticMaskGenerator}, a flat set of object–proposal masks that are pruned only by \emph{box}-level non-maximum suppression, leaving many pixel overlaps and no parent–child ordering among the masks. 
As a result, the masks neither partition the image domain nor form a nested hierarchy \cite{kirillov2023segment}. 
These properties violate the disjoint-coalition and containment assumptions required by the Owen extension of Shapley values, whose recursive additivity hinges on a binary partition tree (BPT). 
Consequently, raw SAM outputs cannot be plugged directly into hierarchical-Shapley frameworks without additional structure.

One direction is the \emph{Explain Any Concept} (EAC), which couples SAM with Shapley values to attribute predictions to a \emph{flat} set of SAM-derived concept masks \cite{sun2023explain}. 
Because these masks may overlap and are treated independently, the Shapley computation is executed \emph{once} on the initial segmentation, with no mechanism for iterative refinement of coalitions. 
This design means EAC’s faithfulness is entirely dependent on the initial SAM proposals already capturing all semantically relevant regions (like LIME), thereby breaking requirement R2 of the ShapBPT framework (i.e. progressive, data-driven refinement of the coalition hierarchy). 
In scenarios where the first-pass segmentation misses fine-grained or contextually important regions, EAC cannot recover them, whereas ShapBPT’s recursive splits can adaptively hone in on such details.

A viable research direction toward a SAM-based Hierarchical Coalition Structure (HCS) is to post-process the SAM mask set into a non-overlapping, nested hierarchy that is compatible with the Owen recursion. 
One potential starting point follows the \emph{Panoptic-SAM} pipeline\footnote{Panoptic Segment Anything. \url{https://github.com/segments-ai/panoptic-segment-anything}}, which ``paints'' SAM masks from largest to smallest to obtain a panoptic segmentation; a region-adjacency graph could then be constructed and merged iteratively (e.g., by similarity or containment) to yield a balanced tree. 
Still the tree is not strictly binary, which either needs a binarization or requires a reformulation of \eqref{eq:BPT} to deal with $n$-ary trees.

We plan as a future work to define a SAM-HCS and test its effectiveness against a morphology-driven BPTs, to understand how effective it is w.r.t. a fully refinable BPT structure.

\subsection{Remarks on h-Shap}

We considered including \emph{h-Shap}\,\cite{teneggi2022hShap}, which has a faster convergence compared to Theorem~\ref{th:cost}. However, since the object recognition task addressed by \emph{h-Shap} is incomparable to that of the other XCV methods, as h-Shap treats binary-valued games only, we chose to exclude it. 
Despite this, we believe that \emph{h-Shap} would also benefit from the use of BPT partitions.


\subsection{ANOVA results}
\label{app:anova-results}

To assess statistical significance, we conducted one-way ANOVA tests for each of the four scores ($\mathit{AUC}^+$, $\mathit{AUC}^-$, $\AUIoU$ and $\maxIoU$) obtained by competing explainers across the full test split.
Scores are computed per image (one sample per image $\times$ explainer). We assume that independence is guaranteed since each image is a different problem.
We test the null hypothesis ($H_0$) of equal means across all sample populations. A $p$-value threshold of $0.05$ was used to determine significance, and $H_0$ is rejected when $p < 0.05$.
In all cases, the null hypothesis was rejected, indicating that the results are statistically significant. 
Table \ref{tab:anova_results} reports the results of the one-way ANOVA tests.

\begin{table*}[h]
\centering
\begin{tabular}{|c|c|c|c|c|c|c|c|}
\hline
\textbf{Experiment} & \textbf{No. of Methods} & \textbf{No. of Images} & \textbf{$\textbf{\textit{AUC}}^+$} & \textbf{$\textbf{\textit{AUC}}^-$} & \textbf{\textit{max-IoU}} & \textbf{\textit{AU-IoU}} \\
\hline
E1 & 14 & 574 & 0.0 & 4.10e-197 & 0.0 & 2.49e-135 \\
E2 & 12 & 574 & 0.0 & 0.0       & 0.0 & 0.0       \\
E3 & 14 & 621 & 0.0 & 4.56e-248 & 0.0 & 0.0       \\
E4 & 9  & 274 & 1.09e-111 & 1.54e-15 & 1.94e-96 & 1.29e-18 \\
E5 & 14 & 436 & 0.0 & 0.0       & 1.36e-31 & 1.64e-13 \\
E6 & 9  & 280 & 2.12e-145 & 4.03e-175 & 8.53e-260 & 0.009 \\
E7 & 13 & 593 &  0.0 & 1.27e-209 & 6.06e-162 & 0.0 \\

\hline
\end{tabular}
\caption{One-way ANOVA summary of all four metrics across the seven experiments.}
\label{tab:anova_results}
\end{table*}



\subsection{Convergence analysis}

All experiments presented were performed on a fixed budget of evaluations. This does not clarify the convergence rate of BPT w.r.t. simpler strategies such as AA.
We repeat experiment \textit{E2} (ideal model, 574 images) with varying budgets from 100 to 2000, in increments of 100 evaluations, and plot the progression of the average scores.
Results are shown in Fig. \ref{fig:convergence}. 

\begin{figure}[H]
    \centering
    \includegraphics[width=0.8\linewidth]{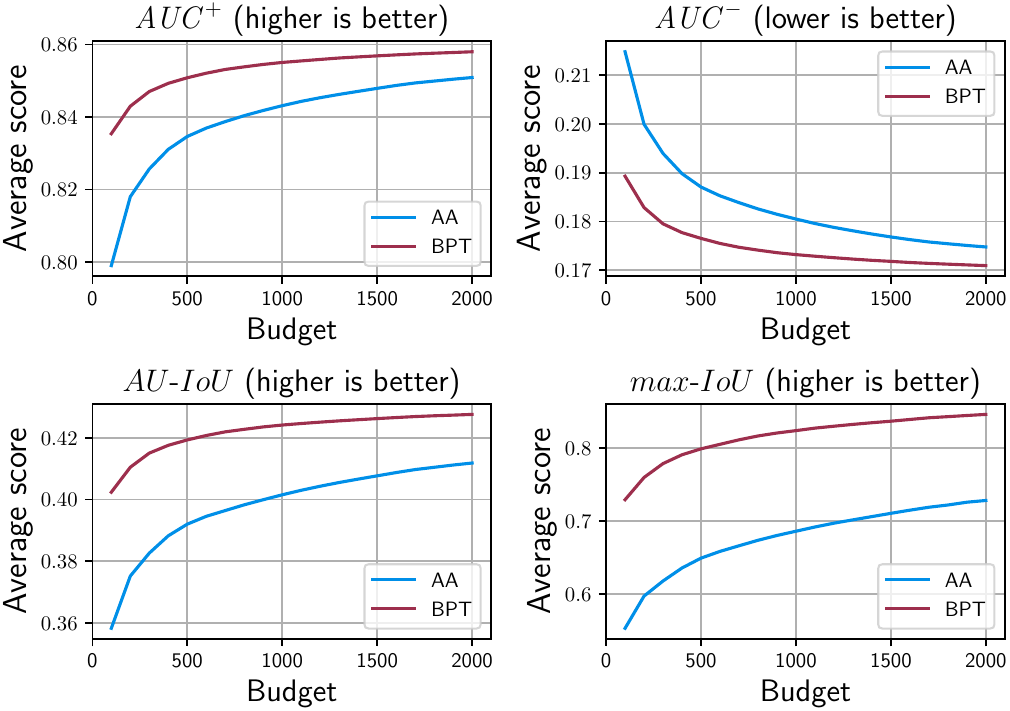}
    \caption{Results for convergence analysis}
    \label{fig:convergence}
\end{figure}

Plots (Fig.~\ref{fig:convergence}) confirm the general intuition of the paper: BPT performs as an accelerator for the Shapley coefficient computation. To achieve similar results with AA a much larger budget is needed.

%% file: references.bib
@article{shapley1953value,
  title={A value for n-person games},
  author={Shapley, Lloyd S},
  year={1953},
  journal={The Shapley value. Essays in honor of Lloyd S. Shapley},
  pages={31},
  publisher={Princeton University Press Princeton}
}

@inproceedings{owen1977values,
  title={Values of games with a priori unions},
  author={Owen, Guilliermo},
  booktitle={Mathematical economics and game theory: Essays in honor of Oskar Morgenstern},
  pages={76--88},
  year={1977},
  organization={Springer}
}

@book{owen2013book,
  title={Game theory, 4th Ed.},
  author={Owen, Guillermo},
  year={2013},
  publisher={Emerald Group Publishing}
}

@article{lopez2009relationship,
  title={{On the relationship between Shapley and Owen values}},
  author={L{\'o}pez, Susana and Saboya, Martha},
  journal={Central European Journal of Operations Research},
  volume={17},
  pages={415--423},
  year={2009},
  publisher={Springer}
}

@article{dubey1981semivalues,
  title={Value theory without efficiency},
  author={Dubey, Pradeep and Neyman, Abraham and Weber, Robert James},
  journal={Mathematics of Operations Research},
  volume={6},
  number={1},
  pages={122--128},
  year={1981},
  publisher={INFORMS}
}

@inproceedings{ijcai2022ShapleyInML,
  title     = {{The Shapley Value in Machine Learning}},
  author    = {Rozemberczki, Benedek and Watson, Lauren and Bayer, Péter and Yang, Hao-Tsung and Kiss, Olivér and Nilsson, Sebastian and Sarkar, Rik},
  booktitle = {{IJCAI-22}},
  pages     = {5572--5579},
  year      = {2022},
}

@inproceedings{lundberg2017unified,
    title = {A Unified Approach to Interpreting Model Predictions},
    author = {Lundberg, Scott M and Lee, Su-In},
    booktitle = {Advances in Neural Information Processing Systems 30},
    pages = {4765--4774},
    year = {2017},
}

@inproceedings{shrikumar2017DeepLIFT,
  title={Learning important features through propagating activation differences},
  author={Shrikumar, Avanti and Greenside, Peyton and Kundaje, Anshul},
  booktitle={International conference on machine learning},
  pages={3145--3153},
  year={2017},
  organization={PMLR}
}

@misc{shapPartitionExplainer,
  author = {Scott Lundberg},
  title = {{The SHAP Partition Explainer}},
  howpublished = "{\small\url{https://shap.readthedocs.io/en/latest/generated/shap.PartitionExplainer.html}}",
  note = {Accessed on 2025-Nov-28},
  year = {2020}, 
}

@article{randrianasoa2018multiFeatureBPT,
  title={Binary partition tree construction from multiple features for image segmentation},
  author={Randrianasoa, Jimmy Francky and Kurtz, Camille and Desjardin, Eric and Passat, Nicolas},
  journal={Pattern Recognition},
  volume={84},
  pages={237--250},
  year={2018},
  publisher={Elsevier}
}

@article{randrianasoa2021agatBPT,
  title={{AGAT: Building and evaluating binary partition trees for image segmentation}},
  author={Randrianasoa, Jimmy Francky and Kurtz, Camille and Desjardin, Eric and Passat, Nicolas},
  journal={SoftwareX},
  volume={16},
  pages={100855},
  year={2021},
  publisher={Elsevier}
}

@article{salembier2000BPT,
  title={Binary partition tree as an efficient representation for image processing, segmentation, and information retrieval},
  author={Salembier, Philippe and Garrido, Luis},
  journal={IEEE Trans. on Image Processing},
  volume={9},
  number={4},
  pages={561--576},
  year={2000},
  publisher={IEEE}
}

@article{teneggi2022hShap,
  title={Fast hierarchical games for image explanations},
  author={Teneggi, Jacopo and Luster, Alexandre and Sulam, Jeremias},
  journal={IEEE Transactions on Pattern Analysis and Machine Intelligence},
  volume={45},
  number={4},
  pages={4494--4503},
  year={2022},
  publisher={IEEE}
}

@inproceedings{ribeiro2016lime,
  title={{"Why should I trust you?" Explaining the predictions of any classifier}},
  author={Ribeiro, Marco Tulio and Singh, Sameer and Guestrin, Carlos},
  booktitle={Proc. ACM SIGKDD Int. Conf., 22nd},
  pages={1135--1144},
  year={2016}
}

@inproceedings{riseScores,
  author       = {Vitali Petsiuk and
                  Abir Das and
                  Kate Saenko},
  title        = {{RISE: Randomized Input Sampling for Explanation of Black-box Models}},
  booktitle    = {British Machine Vision Conference (BMVC) 2018},
  pages        = {151},
  publisher    = {{BMVA} Press},
  year         = {2018},
}

@misc{gradcampytorch,
  title={{PyTorch library for CAM methods}},
  author={Jacob Gildenblat and contributors},
  year={2021},
  publisher={GitHub},
  howpublished={{\small\url{https://github.com/jacobgil/pytorch-grad-cam}}},
  note = {Accessed on 2025-Nov-28},
}

@article{GANGOPADHYAY202364,
    title = {Benchmarking framework for anomaly localization: Towards real-world deployment of automated visual inspection},
    journal = {Journal of Manufacturing Systems},
    volume = {69},
    pages = {64-75},
    year = {2023},
    author = {Tryambak Gangopadhyay and Sungmin Hong and Sujoy Roy and Yash Shah and Lin Lee Cheong},
}

@article{walker2024integrated, 
    title={{Integrated Decision Gradients: Compute Your Attributions Where the Model Makes Its Decision}}, 
    volume={38}, 
    xxurl={https://ojs.aaai.org/index.php/AAAI/article/view/28336}, 
    xxDOI={10.1609/aaai.v38i6.28336}, 
    number={6}, 
    journal={AAAI}, 
    author={Walker, Chase and Jha, Sumit and Chen, Kenny and Ewetz, Rickard}, 
    year={2024}, 
    pages={5289-5297} 
}

@article{gao2022lussImageNetS,
  title={Large-scale Unsupervised Semantic Segmentation},
  author={Gao, Shanghua and Li, Zhong-Yu and Yang, Ming-Hsuan and Cheng, Ming-Ming and Han, Junwei and Torr, Philip},
  journal="TPAMI",
  year={2022}
}

@article{covert2021explaining,
  title={Explaining by removing: A unified framework for model explanation},
  author={Covert, Ian and Lundberg, Scott and Lee, Su-In},
  journal={Journal of Machine Learning Research},
  volume={22},
  number={209},
  pages={1--90},
  year={2021}
}

@article{CLEVRXAI202214,
    title = {{CLEVR-XAI}: A benchmark dataset for the ground truth evaluation of neural network explanations},
    journal = {Information Fusion},
    volume = {81},
    pages = {14-40},
    year = {2022},
    issn = {1566-2535},
    author = {Leila Arras and Ahmed Osman and Wojciech Samek},
}

@misc{vryniotis2021train,
  title={How to train state-of-the-art models using torchvision’s latest primitives},
  author={Vryniotis, Vasilis},
  year={2021},
  howpublished={{\small\url{https://pytorch.org/blog/how-to-train-state-of-the-art-models-using-torchvision-latest-primitives/}}},
  note = {Accessed on 2025-Nov-28},
}

@article{ILSVRC15,
    Author = {Olga Russakovsky and Jia Deng and Hao Su and Jonathan Krause and Sanjeev Satheesh and Sean Ma and Zhiheng Huang and Andrej Karpathy and Aditya Khosla and Michael Bernstein and Alexander C. Berg and Li Fei-Fei},
    Title = {{ImageNet Large Scale Visual Recognition Challenge}},
    Year = {2015},
    journal   = {Int. J. of Computer Vision (IJCV)},
    volume={115},
    number={3},
    pages={211-252}
}

@inproceedings{resnet50,
  title={Deep residual learning for image recognition},
  author={He, Kaiming and Zhang, Xiangyu and Ren, Shaoqing and Sun, Jian},
  booktitle={Proceedings of the IEEE conference on computer vision and pattern recognition},
  pages={770--778},
  year={2016}
}

@inproceedings{karras2018progressive,
    title = "{Progressive Growing of GANs for Improved Quality, Stability, and Variation}",
    author = "Tero Karras and Timo Aila and Samuli Laine and Jaakko Lehtinen",
    year = "2018",
    language = "English",
    booktitle = "Proceedings of International Conference on Learning Representations (ICLR) 2018"
}

@inproceedings{bergmann2019mvtec,
  title={{MVTec AD--A comprehensive real-world dataset for unsupervised anomaly detection}},
  author={Bergmann, Paul and Fauser, Michael and Sattlegger, David and Steger, Carsten},
  booktitle={Proceedings of the IEEE/CVF conference on computer vision and pattern recognition},
  pages={9592--9600},
  year={2019}
}

@inproceedings{ravi2021general,
  title={General frameworks for anomaly detection explainability: comparative study},
  author={Ravi, Ambareesh and Yu, Xiaozhuo and Santelices, Iara and Karray, Fakhri and Fidan, Baris},
  booktitle={2021 IEEE International Conference on Autonomous Systems (ICAS)},
  pages={1--5},
  year={2021},
  organization={IEEE}
}

@inproceedings{liu2021swin,
  title={{Swin transformer: Hierarchical vision transformer using shifted windows}},
  author={Liu, Ze and Lin, Yutong and Cao, Yue and Hu, Han and Wei, Yixuan and Zhang, Zheng and Lin, Stephen and Guo, Baining},
  booktitle={Proceedings of the IEEE/CVF international conference on computer vision},
  pages={10012--10022},
  year={2021}
}

@misc{Kartik,
  title = {{MultiLabel Classification of CelebA}},
  author={Kartik Batra},
  howpublished = {{\small\url{https://www.kaggle.com/code/kartikbatra/multilabelclassification/output}}},
  note = {Accessed on 2025-Nov-28},
  year = 2020
}

@inproceedings{sundararajan2017axiomatic,
  title={Axiomatic attribution for deep networks},
  author={Sundararajan, Mukund and Taly, Ankur and Yan, Qiqi},
  booktitle={International conference on machine learning},
  pages={3319--3328},
  year={2017},
  organization={PMLR}
}

@article{bach2015pixel,
  title={On pixel-wise explanations for non-linear classifier decisions by layer-wise relevance propagation},
  author={Bach, Sebastian and Binder, Alexander and Montavon, Gr{\'e}goire and Klauschen, Frederick and M{\"u}ller, Klaus-Robert and Samek, Wojciech},
  journal={PloS one},
  volume={10},
  number={7},
  pages={e0130140},
  year={2015},
  publisher={Public Library of Science San Francisco, CA USA}
}

@inproceedings{ancona2018towards,
  title={{Towards better understanding of gradient-based attribution methods for Deep Neural Networks}},
  author={Ancona, Marco and Ceolini, Enea and {\"O}ztireli, Cengiz and Gross, Markus},
  booktitle={6th International Conference on Learning Representations (ICLR)},
  year={2018}
}

@inproceedings{MSCOCO,
  title={Microsoft {COCO}: Common objects in context},
  author={Lin, Tsung-Yi and Maire, Michael and Belongie, Serge and Hays, James and Perona, Pietro and Ramanan, Deva and Doll{\'a}r, Piotr and Zitnick, C Lawrence},
  booktitle={In procs. of 13th European Conf. Computer Vision (ECCV) 2014},
  pages={740--755},
  year={2014},
  organization={Springer}
}

@misc{yolo11ultralytics,
  author = {Glenn Jocher and Jing Qiu},
  title = {Ultralytics YOLO11},
  version = {11.0.0},
  year = {2024},
  url = {https://github.com/ultralytics/ultralytics},
  note = {Accessed on 2025-Nov-28},
}

@InProceedings{Englebert_2023_ICCV,
    author    = {Englebert, Alexandre and Stassin, S\'edrick and Nanfack, G\'eraldin and Mahmoudi, Sidi Ahmed and Siebert, Xavier and Cornu, Olivier and De Vleeschouwer, Christophe},
    title     = {Explaining Through Transformer Input Sampling},
    booktitle = {Proceedings of the IEEE/CVF International Conference on Computer Vision (ICCV) Workshops},
    month     = {October},
    year      = {2023},
    pages     = {806-815}
}

@article{hama2023deletion,
  title={Deletion and insertion tests in regression models},
  author={Hama, Naofumi and Mase, Masayoshi and Owen, Art B},
  journal={Journal of Machine Learning Research},
  volume={24},
  number={290},
  pages={1--38},
  year={2023}
}

@inproceedings{rezatofighi2019generalized,
  title={Generalized intersection over union: A metric and a loss for bounding box regression},
  author={Rezatofighi, Hamid and Tsoi, Nathan and Gwak, JunYoung and Sadeghian, Amir and Reid, Ian and Savarese, Silvio},
  booktitle={Proceedings of the IEEE/CVF conference on computer vision and pattern recognition},
  pages={658--666},
  year={2019}
}

@inproceedings{knab2025beyond,
  title={Beyond Pixels: Enhancing {LIME} with Hierarchical Features and Segmentation Foundation Models},
  author={Knab, Patrick and Marton, Sascha and Bartelt, Christian},
  booktitle={ICLR 2025 Workshop on Foundation Models in the Wild},
  year={2025}
}

@inproceedings{kirillov2023segment,
  title={Segment anything},
  author={Kirillov, Alexander and Mintun, Eric and Ravi, Nikhila and Mao, Hanzi and Rolland, Chloe and Gustafson, Laura and Xiao, Tete and Whitehead, Spencer and Berg, Alexander C and Lo, Wan-Yen and others},
  booktitle={Proceedings of the IEEE/CVF international conference on computer vision},
  pages={4015--4026},
  year={2023}
}

@article{sun2023explain,
  title={Explain any concept: Segment anything meets concept-based explanation},
  author={Sun, Ao and Ma, Pingchuan and Yuan, Yuanyuan and Wang, Shuai},
  journal={Advances in Neural Information Processing Systems},
  volume={36},
  pages={21826--21840},
  year={2023}
}

@inproceedings{dosovitskiyimage,
  title={An Image is Worth 16x16 Words: Transformers for Image Recognition at Scale},
  author={Dosovitskiy, Alexey and Beyer, Lucas and Kolesnikov, Alexander and Weissenborn, Dirk and Zhai, Xiaohua and Unterthiner, Thomas and Dehghani, Mostafa and Minderer, Matthias and Heigold, Georg and Gelly, Sylvain and others},
  booktitle={International Conference on Learning Representations},
  year={2020}
}

@article{deng1994complexity,
  title={On the complexity of cooperative solution concepts},
  author={Deng, Xiaotie and Papadimitriou, Christos H},
  journal={Mathematics of operations research},
  volume={19},
  number={2},
  pages={257--266},
  year={1994},
  publisher={INFORMS}
}
